\documentclass{article}

\usepackage{microtype}
\usepackage{graphicx}
\usepackage{subcaption}
\usepackage{booktabs} 

\usepackage{hyperref}

\usepackage[
  detect-all,
  per-mode=symbol
]{siunitx}

\usepackage[preprint]{icml2026}

\usepackage{amsmath}
\usepackage{amssymb}
\usepackage{mathtools}
\usepackage{amsthm}

\usepackage[capitalize,noabbrev]{cleveref}

\newcommand{\Sph}[0]{\mathcal{S}}
\newcommand{\EE}[0]{\mathbb{E}}

\newcommand{\MM}[0]{\mathcal{M}}
\newcommand{\SO}[0]{\mathrm{SO}}

\theoremstyle{plain}
\newtheorem{theorem}{Theorem}[section]
\newtheorem{proposition}[theorem]{Proposition}

\newtheorem{corollary}[theorem]{Corollary}
\theoremstyle{definition}

\theoremstyle{remark}

\usepackage[textsize=tiny]{todonotes}

\icmltitlerunning{Manifold Aware Denoising Score Matching (MAD)}

\begin{document}

\twocolumn[
  \icmltitle{Manifold Aware Denoising Score Matching (MAD)}



  \icmlsetsymbol{equal}{*}

  \begin{icmlauthorlist}
    \icmlauthor{Alona Levy-Jurgenson}{yyy}
    \icmlauthor{Alvaro Prat}{yyy}
    \icmlauthor{James Cuin}{imper}
    \icmlauthor{Yee Whye Teh}{yyy}
  \end{icmlauthorlist}

  \icmlaffiliation{yyy}{Department of Statistics
University of Oxford}
  \icmlaffiliation{imper}{Department of Mathematics, Imperial College London, London, United Kingdom}

  \icmlcorrespondingauthor{Alona Levy-Jurgenson}{alona.jurgenson@stats.ox.ac.uk}
  \icmlcorrespondingauthor{Yee Whye Teh}{y.w.teh@stats.ox.ac.uk}

  \icmlkeywords{denoising score matching, manifold, diffusion, generative models, deep learning, machine learning}

  \vskip 0.3in
]



\printAffiliationsAndNotice{}  

\begin{abstract} 
A major focus in designing methods for learning distributions defined on manifolds is to alleviate the need to implicitly learn the manifold so that learning can concentrate on the data distribution within the manifold. However, accomplishing this often leads to compute-intensive solutions.
In this work, we propose a simple modification to denoising score-matching in the ambient space to implicitly account for the manifold, thereby reducing the burden of learning the manifold while maintaining computational efficiency. Specifically, we propose a simple decomposition of the score function into a known component $s^{base}$ and a remainder component $s-s^{base}$ (the learning target), with the former implicitly including information on where the data manifold resides. 
We derive known components $s^{base}$ in analytical form for several important cases, including distributions over rotation matrices and discrete distributions, and use them to demonstrate the utility of this approach in those cases. 
\end{abstract}

\begin{figure*}[ht]
    \centering
    \begin{minipage}{0.24\textwidth}
        \centering
        \includegraphics[width=\linewidth]{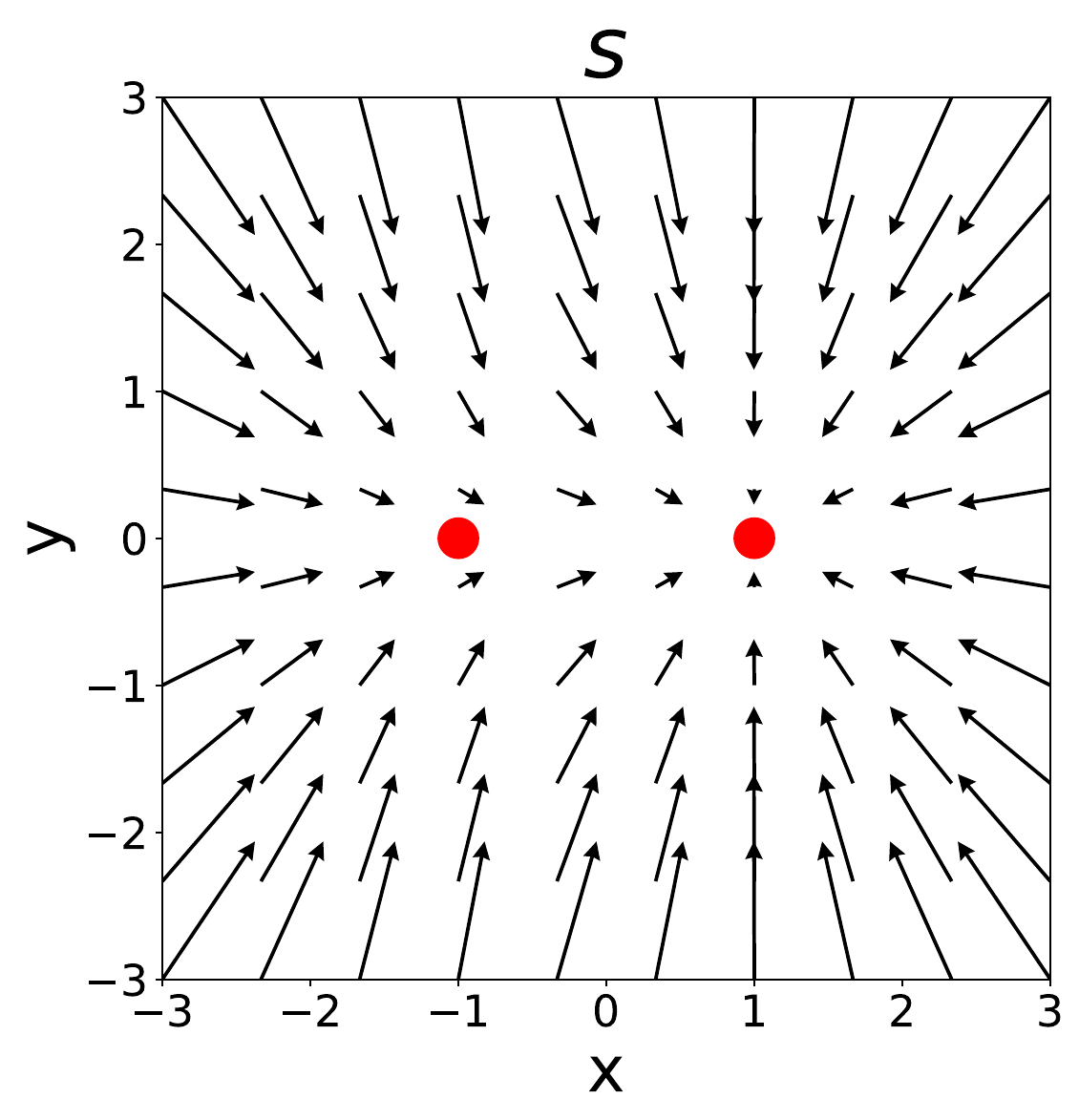}
        \label{fig:fig1}
    \end{minipage}%
    \hfill
    \begin{minipage}{0.24\textwidth}
        \centering
        \includegraphics[width=\linewidth]{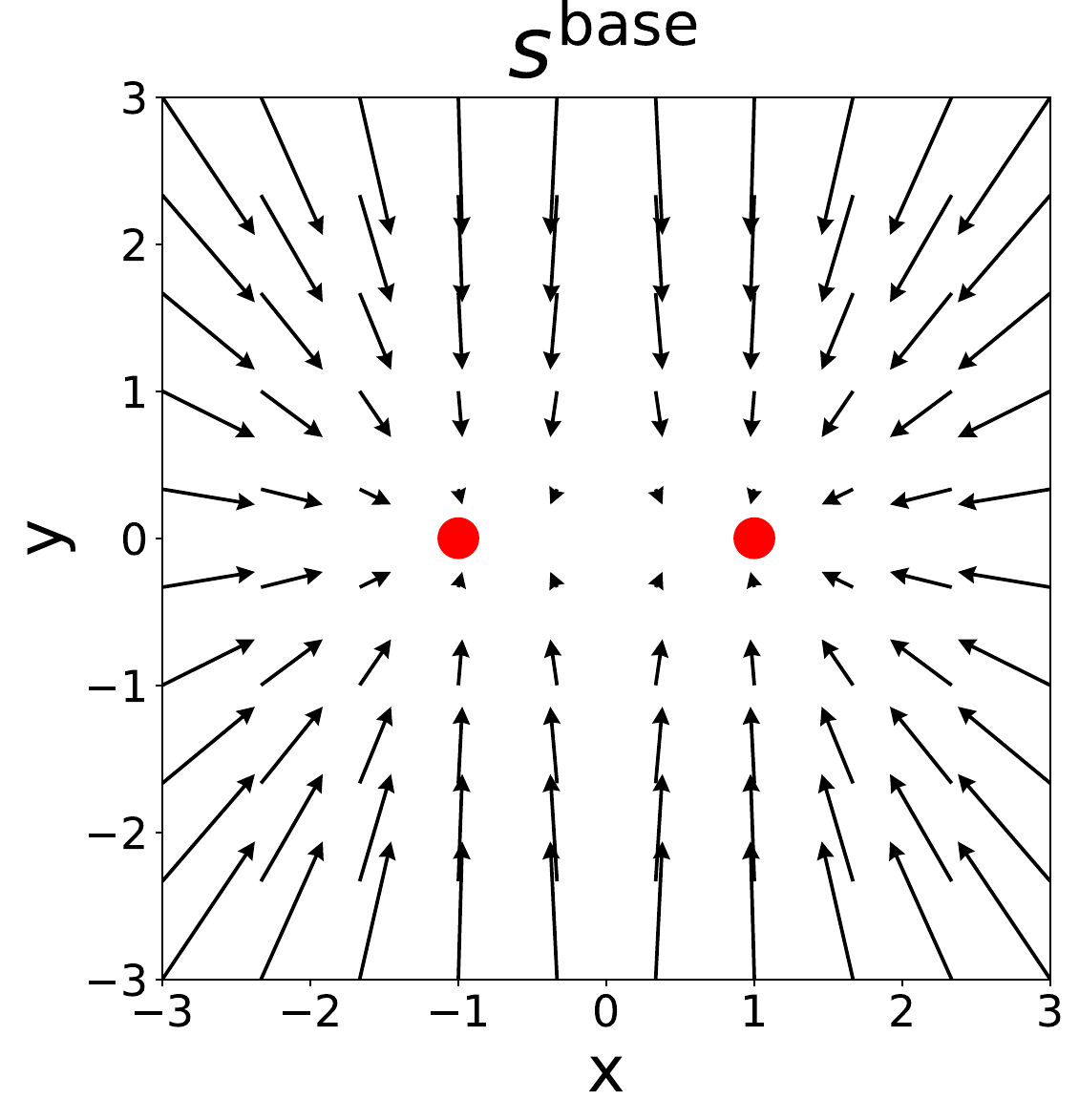}
        \label{fig:fig2}
    \end{minipage}%
    \hfill
    \begin{minipage}{0.24\textwidth}
        \centering
        \includegraphics[width=\linewidth]{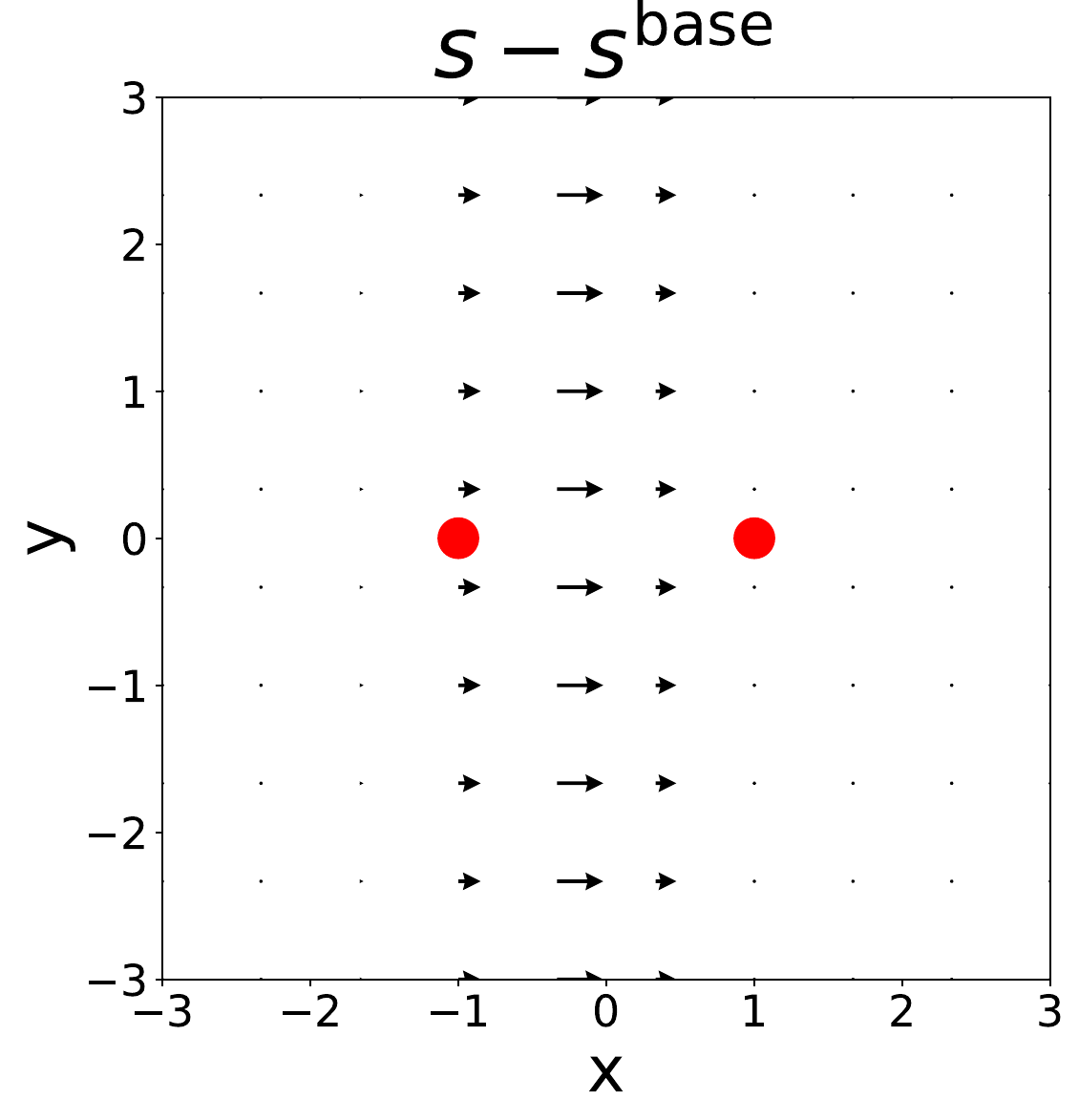}
        \label{fig:fig3}
    \end{minipage}%
    \hfill
    \begin{minipage}{0.24\textwidth}
        \centering
        \includegraphics[width=\linewidth]{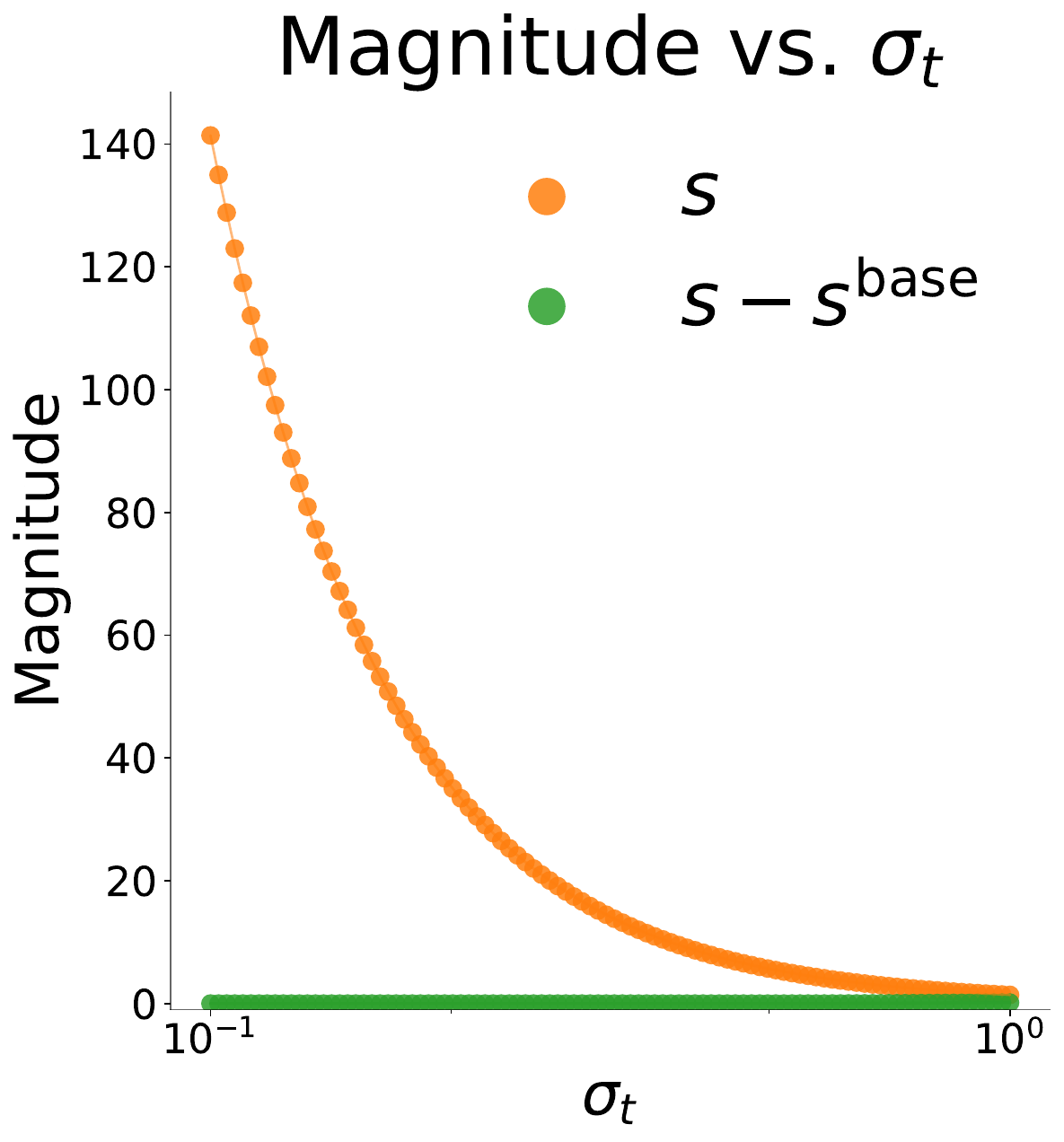}
        \label{fig:fig4}
    \end{minipage}
    \caption{Introduction to MAD through a toy discrete distribution over $\MM=\{(-1,0),(1,0)\}$ (red) 
    with $p = \{0.1, 0.9\}$, respectively, and $p^{base} = \{0.5, 0.5\}$.
    Vector fields are shown for the actual score $s$, the base score $s^{base}$ and the difference $s-s^{base}$ for $\sigma_t = 0.8$. The magnitude of $s$ (DSM's learning target) and $s-s^{base}$ (MAD's learning target) as a function of $\sigma_t$ is shown for $x=(1,0)$.}
    \label{fig:intro}
\end{figure*}

\section{Introduction}
Score-based generative models (SGMs) trained with denoising score matching (DSM) offer an efficient and scalable approach to generative modeling \cite{song2019generative, song2020score}. However, many data modalities of practical interest are supported on low-dimensional manifolds. This includes rotations in $\SO(3)$, which are relevant for drug-design and robotics \cite{senanayake2018directional, zhang2024flexsbdd, prat2025sigmadockuntwistingmoleculardocking}, geology and climate data \cite{karpatne2018machine}. Also of interest is, discrete data, ubiquitous in the context of text generation \cite{shi2024simplified, amin2025masking, feng2025theoretical} and bio-informatics \cite{huang2024latent}, which is supported on lower-dimensional sets embedded in higher-dimensional spaces. In such settings, standard DSM, which assumes full support in the ambient space, is implicitly tasked with simultaneously learning both the manifold and the distribution supported on it. This can result in a more challenging learning problem than in settings where the target distribution has full ambient-space support.

To address this, several approaches construct generative models that explicitly account for manifold structure, for example by defining an on-manifold generative process \cite{rozen2021moser, huang2022riemannian, de2022riemannian, prat2025sigmadockuntwistingmoleculardocking}, where the score is defined through the tangent space of the manifold, or by applying Euclidean generative models in a low-dimensional setting using methods like charts, where the dimension matches the intrinsic dimension of the manifold, followed by a mapping back to the manifold \cite{gemici2016normalizing, kalatzis2021density}. While effective, the former can be computationally intensive during training, often yielding sampling schemes that require fine discretisation to navigate the high curvature of the manifold \cite{de2022riemannian, lou2023scalingriemanniandiffusionmodels}. 
On the other hand, the latter frequently require multiple mappings that may also introduce distortions or dependency on the chosen maps \cite{gemici2016normalizing, brehmer2020flows, kalatzis2021density}.

In contrast, DSM in the ambient space is computationally efficient and straightforward to deploy.
Furthermore, working in Euclidean space relaxes the geometry of the path: instead of being constrained to follow curved on-manifold geodesics, trajectories may take more direct ``chord-like” routes, noting that for embedded manifolds $d_\mathcal{M}(x,y)\geq ||x-y||_2$ \cite{Lee2012}.
However, ambient-space diffusion accounts for the manifold structure only implicitly through the training data. Hence, many ambient-space
probability-flow models have struggled to generate coherent in-manifold structures due to this very issue (e.g. \citep{D3SC04185A}). This lack of information
on the manifold structure affects the learning dynamics -- recent work observed that diffusion models, and specifically DSM, must first recover the support of the data distribution before learning any information about the density \cite{li2025scores}. Together, these finding raise a natural question: can ambient-space DSM be modified to account for the support structure, while retaining its simplicity and improving learning efficiency?

In this work, we answer this question by decomposing the score into two components: a known base score that captures the manifold structure, and an unknown residual term that depends on the target distribution. We term this approach \textit{Manifold Aware Denoising Score Matching} (MAD). A demonstration of our approach is shown in Figure \ref{fig:intro}. We analytically derive base scores for a range of important manifolds, including rotations, n-dimensional spheres, and discrete data. Because the base score is known, learning can focus exclusively on the residual, distribution-dependent, component. We show how MAD can lead to faster convergence and improved distributional fidelity, while achieving performance comparable to, and in some cases better than, more computationally intensive alternatives, without sacrificing the efficiency and simplicity of standard DSM. We also demonstrate how this decomposition can lead to a more stable learning target for score-matching in the discrete case, replacing the need to learn the score function, which diverges as noise levels become smaller. 

\section{Method}

Formally, we are interested in learning a diffusion model for a target distribution that is supported on a subset $\MM \subseteq \mathbb{R}^n$ for some $n \in \mathbb{{N}}_{>0}$. We assume that we have a simple and tractable base probability measure $\mu$ over $\MM$, 
with the target distribution $p$ (with some notational overlap) absolutely continuous with respect to $\mu$. For example, we can take $\MM$ to be a known compact Riemannian manifold with $\mu$ being the normalised and uniform Haar measure on $\MM$, or a finite set with $\mu$ being the uniform counting measure on $\MM$. 

\subsection{Background: Diffusion Models}

A diffusion model is defined by a forward process that adds noise to samples drawn according to $p$, and a reverse process that denoises samples to recover the original distribution. Let $x_t$ be the noisy sample at time $t$, with noise variance $\sigma_t^2$. In our setup, we work with VE-SDEs \cite{song2020score}. Denote $g_t^2 := \frac{d\sigma_t^2}{dt}$ and $dW_t$ as a standard Wiener process in $\mathbb{R}^n$. We have Brownian motion as the forward process:
\begin{align}
    x_0 &\sim p \\
    dx_t &= g_t dW_t
\end{align}
The reverse process is given by a stochastic differential equation (SDE),
\begin{align}
    dx_t = g_t^2 \nabla_{x_t} \log p_t(x_t)dt + g_t dW_t,
\end{align}
where the score function is given by,
\begin{align}
    s(x_t,t) &:= \nabla_{x_t} \log p_t(x_t) \notag \\ 
    &= \nabla_{x_t} \log \int N_{\sigma_t}(x_t-x_0) p(x_0) \mu
    (dx_0) \notag \\
    &= \frac{ \EE[x_0|x_t] - x_t}{\sigma_t^2} \label{eq:score_fn}\\ 
    N_{\sigma_t}(x_t-x_0) &:= \frac{1}{(2\pi\sigma_t^2)^{d/2}} \exp\left(-\frac{\|x_t-x_0\|^2}{2\sigma_t^2}\right)
\end{align}

Typically, a neural network is used to approximate the score function, $s(x_t,t) \approx s_\theta(x_t,t)$, with $\theta$ often learned via score matching:
\begin{align}
    \mathcal{L}(\theta) &= \EE_{x_0,x_t,t} \left[\left\| 
   {\sigma_t} s_\theta(x_t,t)-\frac{x_0-x_t}{\sigma_t} \right\|^2_2 \right].
    \label{eq:loss}
\end{align}
where, following \citep[][Section 4.2]{song2019generative}, the score matching loss is weighted by $\sigma_t^2$ for numerical reasons.

\subsection{Score Decomposition}

The key idea behind our approach is to decompose the time-dependent score function $s(x_t,t)$ defined in the ambient space, into two components: 
\begin{align}
    s(x_t,t) &= s^{base}(x_t,t) + \delta(x_t,t)
\end{align}
The first term $s^{base}$ is the score function for the simple base distribution $\mu$, while the second term $\delta$ is the unknown component that is our learning target. We will assume that $s^{base}$ is known and analytically tractable, deriving this explicitly for a number of examples below. Because $\mu$ is uniform and supported on $\MM$, $s^{base}$ effectively captures the geometric structure of $\MM$. As a consequence, $\delta$ effectively captures the properties of the target distribution $p$ supported on $\MM$, without having to implicitly encode the geometric structure of $\MM$. 

A natural question is why learning $\delta$ can be easier than learning the original score function $s$, which we find to be experimentally the case. Intuitively, the score function $s$ captures both the geometry of $\MM$ and the target distribution $p$, while $\delta$ only captures $p$ since the geometry has been captured by $s^{base}$, so learning $\delta$ should be easier since less needs to be learnt.

Another useful intuition is as follows: $s(x_t,t)$ should be well-behaved throughout most of its domain, since the  distribution of the noised data is smooth for $t>0$. The exception is that $s(x_t,t)$ can take on large values when $t\approx 0$ (equivalently $\sigma_t\approx 0$). This is because the score function needs to force $x_t$ to converge onto $\MM$ in the reverse SDE, as $t\rightarrow 0$. This holds true for $s^{base}$ as well; in fact both $s(x_t,t)$ and $s^{base}(x_t,t)$ should act similarly in forcing $x_t$ to converge to the same closest point on $\MM$, if $p$ looks similar to $\mu$ locally. As a result, $\delta(x_t,t)$, which is the difference between $s(x_t,t)$ and $s^{base}(x_t,t)$, should have values which are much smaller than either when $t\approx 0$.

We can show this mathematically in the simple case of a discrete distribution:
\begin{theorem}
\label{thm:o(1)_discrete}
    Let $N \in \mathbb{N}$, $N>1$ and $\MM=\{u_i\}_{i=1}^{N}\subset\mathbb{R}^n$ for some $n \in\mathbb{N}$. Let $\mu$ be the uniform normalised\footnote{Note that in the caption of Figure \ref{fig:intro} the unnormalised counting measure is used, but an equivalent, less intuitive, form may be obtained for its normalised counterpart.} counting measure on $\mathcal{M}$ and $p$ a density function w.r.t  $\mu$ that is supported on $\MM$. Let $x \in \mathbb{R}^n$ so that  $x$ is not equidistant to all points in $\MM$. Then as $t \to 0$ (equivalently, $\sigma_t \to 0$):
\begin{equation}
    \|s(x,t)-s^{base}(x,t)\| \to 0
\end{equation}
\end{theorem}
The proof is in Appendix \ref{appx:proof_for_discrete}.
Theorem \ref{thm:o(1)_discrete} means that the learning target $\delta(x,t)$ scales as $o(1)$ as $\sigma_t \to 0$. This result, combined with the known discrete base score (provided below in Section \ref{sec:base_scores}), makes it theoretically possible to recover the score function for a given $x$ to within $o(1)$ error as $\sigma_t \to 0$. This is important in light of the observation by \cite{li2025scores} that recovering the data density requires the learned score function to match the true score function to within $o(1)$ error as $\sigma_t \to 0$. 

We can now parameterise $s(x_t,t)$ by parameterising $\delta(x_t,t) \approx \delta_\theta(x_t,t)$, giving:
\begin{align}
    s(x_t,t) &\approx s^{base}(x_t,t) + \delta_\theta(x_t,t)
\end{align}
To learn $\theta$, we can adapt the previous loss \eqref{eq:loss}, giving the following loss:
\begin{align}
&\mathcal{L}(\theta) =\\
& \mathbb{E}_{x_0,x_t,t}\Bigg[
 \Bigg\|
\sigma_t \delta_\theta(x_t,t) -
\left(
\frac{x_0 - x_t}{\sigma_t}
- \sigma_t s^{base}(x_t,t)
\right)
\Bigg\|^2_2
\Bigg] \nonumber
\end{align}



In practice, ambient space DSM learning results in a slightly noisy version of $p$ \citep[][Section 3.1]{song2019generative}, typically denoted $p_{\sigma_{\min}}$, in that we learn $p_{\sigma_{\min}}$ rather than $p$. Therefore, as a final step to produce samples that are guaranteed to be on the manifold, unless stated otherwise (e.g.\ Figure \ref{fig:discrete}), we project the generated samples in the ambient space to the manifold. 

\subsection{Examples}\label{sec:base_scores}

Below we derive several base score functions, which will also be used in the results section. All of the proofs are in Appendix \ref{appx: proofs}. An empirical demonstration of the distribution used for the base score can be seen in Appendix Figure \ref{fig:s2_not_projected} where samples were generated using solely the base score.

\subsubsection{Discrete Distributions}
\begin{proposition}(Base score for discrete distributions) \label{thm:base_discrete}
    For a uniform discrete distribution supported on a finite set of points $\MM = \{u_i\}_{i=1}^N$, 
\begin{align}
    s^{base}(x_t,t) = \frac{\sum_{i=1}^N u_i N_{\sigma_t}(x_t-u_i) }{ \sum_{i=1}^N N_{\sigma_t}(x_t-u_i)}
\end{align}
\end{proposition}

\subsubsection{Distributions on a Sphere}

Below we derive the base score for general $n$-spheres, and specifically $2$-spheres. This will also serve a role in the derivation of a base score for distributions over 3D rotations.

\begin{proposition} (Base score for $n$-spheres) \label{thm:base_sphere}
Assume $n \in \mathbb{N}_{>0}$. Let $\mathcal{M} = \mathcal{S}^n(1) := \{x \in \mathbb{R}^{n+1} : \|x\| = 1\}$ and $\mu$ be the normalised spherical measure on $\mathcal{M}$ (commonly denoted $\sigma^n$). Then for $x_t \neq 0$,
\begin{align}
    &s^{base}(x_t,t)  = -\frac{x_t}{\sigma_t^2} + \frac{1 - n}{2}\frac{x_t}{\|x_t\|^2} + \\ & \frac{1}{2I_{\frac{n-1}{2}}\left(\frac{\|x_t\|}{\sigma_t^2}\right)}\bigg( I_{\frac{n-3}{2}}\bigg(\frac{\|x_t\|}{\sigma_t^2}\bigg)+I_{\frac{n+1}{2}}\bigg(\frac{\|x_t\|}{\sigma_t^2}\bigg)\bigg)\frac{x_t}{\sigma_t^2\|x_t\|} \nonumber
\end{align}
where $I_k$ stands for the modified Bessel function of the first kind of order $k$.
\end{proposition}

From the above result, we can obtain the following simplified expression for distributions over $2$-spheres. 

\begin{corollary} (Base score for $2$-spheres) \label{thm:base_2_sphere}
    For $\mathcal{M} = \mathcal{S}^2(1)=\{x \in \mathbb{R}^{3}:||x||=1\}$ and $\mu$  the usual normalised spherical measure on $\mathcal{S}^2(1)$,
\begin{equation}
\begin{split}
s^{base}(x_t,t) = x_t\left(-\frac{1}{\sigma_t^2} 
-\frac{1}{\|x_t\|^2} +
\frac{1}{\sigma_t^2\|x_t\|\tanh{\frac{\|x_t\|}{\sigma_t^2}}} \right) \\
\end{split}
\end{equation}
\end{corollary}

This also provides some intuition: by rearranging the terms, we can see that:
\begin{equation*}
    s^{base}(x_t,t) = \frac{-x_t + \frac{x_t}{\|x_t\|\tanh{}\frac{\|x_t\|}{\sigma_t^2}}-\frac{\sigma_t^2x_t}{\|x_t\|^2}}{\sigma_t^2} \\
\end{equation*}
so by Equation \ref{eq:score_fn} we can deduce that:
\begin{equation}
    \EE^{base}[x_0|x_t] = \frac{x_t}{\|x_t\|\tanh{}\frac{\|x_t\|}{\sigma_t^2}}-\frac{\sigma_t^2x_t}{\|x_t\|^2}
\end{equation}
so that as $\sigma_t \to 0$ the expected value simply converges towards the perturbed sample projected onto $\mathcal{S}^2$, which makes sense for a uniform distribution.

\subsubsection{Distributions Over 3D Rotations}
\label{sec:rotations}
Another case of interest is rotations in $\SO(3)$ which can be represented as unit vectors in $\mathcal{S}^3$ (formally, quaternions in $\mathbb{R}^4$ -- see Appendix \ref{appx:quat_repres}). In this case, we can use Theorem \ref{thm:base_sphere} with $n=3$ and simplify the expression further to include only modified Bessel functions of the first kind of orders $0$ and $1$, which are more commonly implemented in deep learning frameworks:

\begin{corollary} (Base score for $3$-spheres) \label{thm:base_3_sphere}
For $\mathcal{M} = \mathcal{S}^3(1)=\{x \in \mathbb{R}^{4}:||x||=1\}$ and $\mu$  the usual normalised spherical measure on $\mathcal{S}^3(1)$,
    \begin{align}
    \label{eq:base_so3}
    s^{base}(x_t,t)  =& -\frac{x_t}{\sigma_t^2} + \frac{1 - n}{2}\frac{x_t}{\|x_t\|^2} + \\ 
    & \left( \frac{I_{0}\left(\frac{\|x_t\|}{\sigma_t^2}\right)}{I_{1}\left(\frac{\|x_t\|}{\sigma_t^2}\right)} -\frac{\sigma_t^2}{\|x_t\|}\right)\frac{x_t}{\sigma_t^2\|x_t\|} \nonumber
\end{align}
\end{corollary}
Note that this representation is achieved via a 2-to-1 mapping from $\Sph^3$ to $\SO(3)$, where two antipodal points, $x$ and $-x$, correspond to the same rotation.

\textbf{Parity Equivariance}. To account for the double-cover geometry of the rotation group, the distribution of $\Sph^3$ must be invariant under antipodal symmetry. Hence, $p(x)=p(-x)$ for all $x\in\Sph^3$. Consequently, the score function $s(x_t,t)$ must satisfy parity equivariance: $s(-x_t,t) = -s(x_t,t)$, that is, it is odd with respect to the antipodal point. To enforce this structurally, we can parameterise the residual $\delta_\theta$ via a universal parameterisation for parity-equivariant functions via antisymmetrization \cite{de2022riemannian}:  
\begin{equation}
        \delta(x_t,t) \approx \frac{1}{2} \bigg (f_\theta(x_t,t) - f_\theta(-x_t,t) \bigg)
\end{equation}
where $f_\theta(x_t,t)$ is a neural network. Finally, it is trivial to show that our base score in Eq. ~\eqref{eq:base_so3} is parity equivariant, such that $s^\text{base}(-x_t,t) = -s^\text{base}(x_t,t)$. Hence, our score field defined with our composition $s_\theta(x_t,t):= \delta_\theta(x_t,t) + s^\text{base}(x_t,t)$ is strictly parity equivariant.

\textbf{Quotient-Space Diffusion via canonicalisation.} One critical issue with learning conditional distributions over 3D rotations is \emph{non-identifiability} in the presence of rotational symmetries. In particular, the posterior $p(R\mid c)$, where $c$ is any conditioning variable (e.g.\ an image or a molecular graph), is generally \emph{multimodal}: multiple distinct rotations correspond to the same physical state and therefore to the same observation. Formally, for an object with discrete symmetry group $G\subset \SO(3)$, the observation map is invariant under the right action of $G$, i.e.\ $c(R)=c(Rg)$ for all $g\in G$. In the idealized noiseless limit, the posterior concentrates on the symmetry orbit of a ground-truth pose and can be written as a uniform mixture of Dirac measures,
\begin{equation}
    p(R\mid c) \;=\; \frac{1}{|G|}\sum_{g\in G}\delta\bigl(R_{\mathrm{gt}}\,g\bigr),
\end{equation} 
where $R_{\mathrm{gt}}\in \SO(3)$ is any representative consistent with $c$, and the physically meaningful pose is the equivalence class $[R_{\mathrm{gt}}]:=\{R_{\mathrm{gt}}g:g\in G\}$.

Standard DSM trains a network to approximate the conditional score $s(R_t,t\mid c)=\nabla_{R_t}\log p_t(R_t\mid c)$. When $p(R\mid c)$ is multimodal, diffusion trajectories originating from different symmetry-related modes overlap at high noise levels, so the same noisy state $R_t$ can correspond to multiple incompatible denoising targets. In an ambient parameterisation (e.g.\ quaternions in $\mathbb{R}^4$), this induces gradient conflict and drives the model toward Euclidean averages of the modes; such averages typically lie in low-density regions and need not correspond to a valid rotation, yielding degenerate \emph{ghost rotations}.

The root cause is that the physically meaningful variable is not $R\in \SO(3)$, but the equivalence class $[R]\in \SO(3)/G$. We therefore reformulate the problem on the quotient space $\mathcal{M}_q=\SO(3)/G$, on which each symmetry orbit collapses to a single point and the conditional distribution becomes unimodal. In practice, we implement this by canonicalising each ground-truth pose \emph{before} noise injection, i.e.\ by selecting a deterministic representative of the orbit $\{Rg:g\in G\}$ in a fixed fundamental domain. A concrete quaternion-based canonicalisation rule and the associated lifting procedure are given in Appendix~\ref{appx:canon}.



\section{Results}
We evaluate our approach on several benchmarks and compare to both on-manifold as well as ambient-space methods. To obtain a controlled evaluation, in each task we use the same set of training samples and run each method for the same number of steps with the same batch size and the same (or similar) neural network architecture. We evaluate all methods using a single sample-based metric for manifolds -- Maximum Mean Discrepancy (MMD) \cite{mehraban2025quantization}, with the same approach as in \cite{de2022riemannian}, to avoid comparing between different log-likelihood approximation methods for manifolds, and to avoid related log-likelihood pathologies for ambient-space methods (see \cite{loaiza2024deep}).
Full data and training specifications are in Appendix \ref{app:training}.

\textbf{Summary of results. } Our results reveal three main findings: (1) both our method and DSM achieve similar, and in some cases better, MMD compared to other methods, while converging and sampling faster 
(Tables \ref{tab:earth}, \ref{tab:so3}, \ref{tab:symsol}), 
showing that ambient-space methods are strong competitors in this domain; (2) Compared to DSM, MAD leads to faster convergence (Figure \ref{fig:so3_mmd_vs_time}) and consistently lower loss (Appendix Figure \ref{fig:loss_dsm_vs_mad}) while maintaining simplicity, showing that ambient-space DSM can indeed be effectively modified to incorporate information on the geometry of the support so that learning can focus on the distribution within it (e.g. Fig. \ref{fig:earth_closeup_data}); and (3) in the discrete experiments, MAD is able recover the true distribution, as may be anticiapted from Theorem \ref{thm:o(1)_discrete}, while DSM struggles to do so, often generating out-of-distribution samples.
These findings highlight the practical benefits of MAD in low-dimensional settings.

\begin{table*}[h]
    \centering \caption{Earth MMD results (mean ± std) computed using the minimum between 1000 samples and the size of the available test set.} \label{tab:earth}
    \begin{small}
    \begin{sc}
    \begin{tabular}{l c c c c}
        \toprule
        Method & Volcano & Earthquake & Flood & Fire \\
        \midrule
        RSGM-divfree &0.1272 ± 0.0289 & 0.0602 ± 0.0101 & 0.0579 ± 0.0094 & 0.0779 ± 0.0172 \\
        RSGM-ambient & 0.1114 ± 0.0090 & 0.0803 ± 0.0067 & 0.0751 ± 0.0087 & 0.0830 ± 0.0152 \\
        DSM         & 0.1082 ± 0.0217 & 0.0488 ± 0.0096 & 0.0503 ± 0.0092 & 0.0456 ± 0.0165 \\
        MAD (ours)   & 0.1030 ± 0.0185 & 0.0526 ± 0.0108 & 0.0500 ± 0.0100 & 0.0452 ± 0.0131 \\
        \midrule
        Data size & 827 & 6120 & 4875 & 12809 \\
        \bottomrule
        \bottomrule
    \end{tabular}
    \end{sc}
    \end{small}
\end{table*}

\begin{table*}[h]\caption{$\SO(3)$ MMD (mean ± std) computed between generated sets and test set using 5,000 samples, and sampling time for 1K samples. }\label{tab:so3}
    \centering
    \begin{small}
    \begin{sc}
    \begin{tabular}{l c c c c  }
        \toprule
        Method & $K=16$ & $K=32$ & $K=64$ & Sample time (s/1k) \\
        \midrule
            FFF & 0.0230 ± 0.0023 & 0.0242 ± 0.0016 & 0.0204 ± 0.0018 &   0.0040 ± 0.0003\\ 
            RSGM & 0.0155 ± 0.0025 & 0.0171 ± 0.0024 & 0.0190 ± 0.0043 &   0.4628 ± 0.0190\\ 
            DSM & 0.0222 ± 0.0021 & 0.0209 ± 0.0046 &  0.0178 ± 0.0031 &   0.2216 ± 0.0350\\ 
            MAD (ours) & 0.0229 ± 0.0026 & 0.0220 ± 0.0052 & 0.0180 ± 0.0029 &  0.2305 ± 0.0360\\
        \bottomrule
        \bottomrule
    \end{tabular} 
    \end{sc}
    \end{small}
\end{table*}

\subsection{Earth data in $\mathcal{S}^2$}
We begin with earth data represented as points in $\mathcal{S}^2$ (Appendix \ref{appx:datasets_s2}). Results are depicted in Figure \ref{fig:earth} and Table \ref{tab:earth}. From Table \ref{tab:earth} we observe that all methods have either comparable MMDs (e.g. volcano, the smallest dataset), or that MAD and DSM perform slightly better (e.g. Fire, the largest dataset). 
Overall, earth experiments confirm that MAD can yield improved MMD compared to on-manifold methods and can match DSM while capturing sharper distributional detail (Figure \ref{fig:earth_closeup_data}-\ref{fig:earth_closeup_dsm}). 

\begin{figure}[h]
    \centering
    \begin{subfigure}[b]{0.14\textwidth}
        \includegraphics[width=\textwidth]{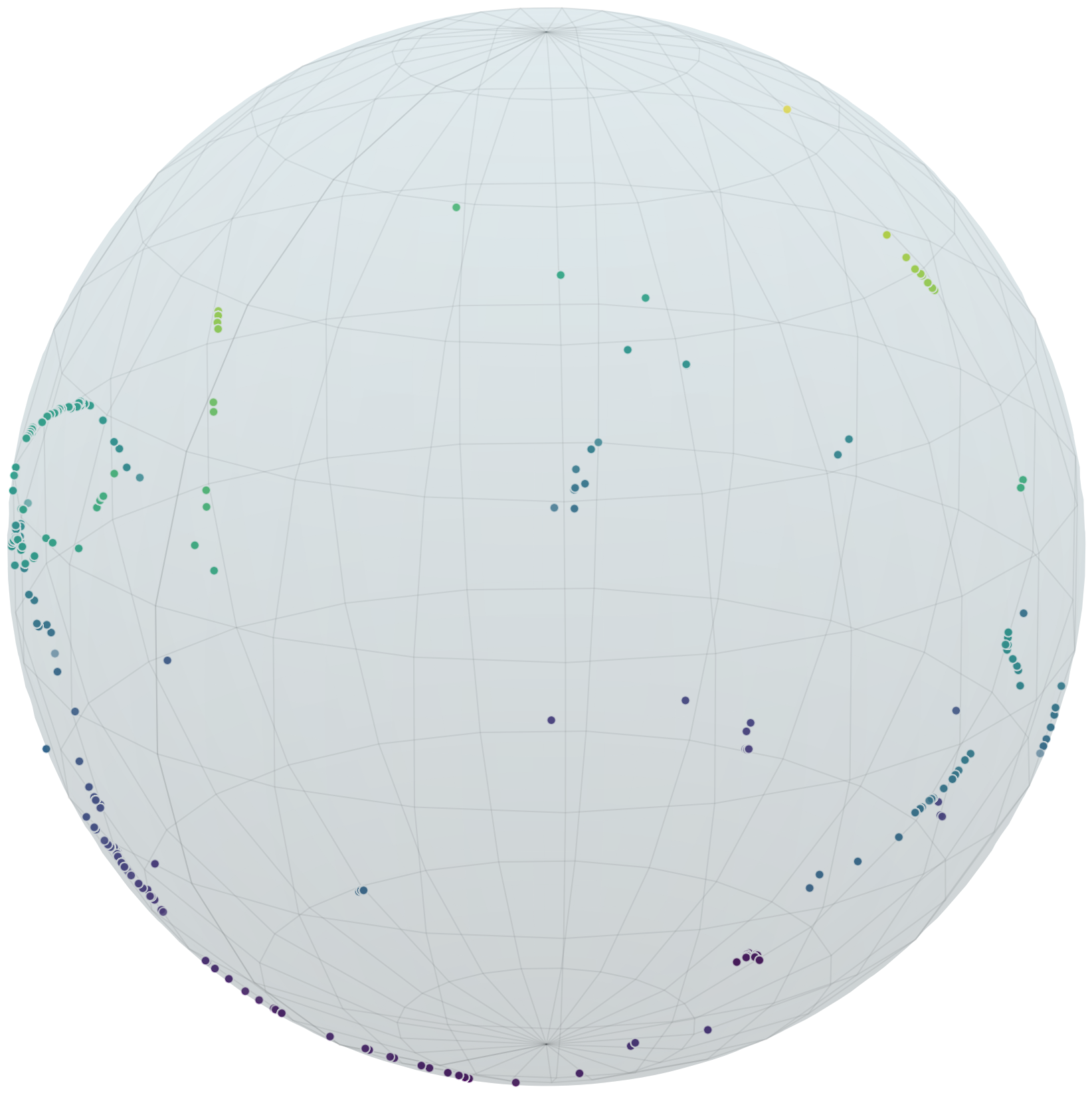}
        \caption{Volcano data}
    \end{subfigure}
    \hfill
    \begin{subfigure}[b]{0.14\textwidth}
        \includegraphics[width=\textwidth]{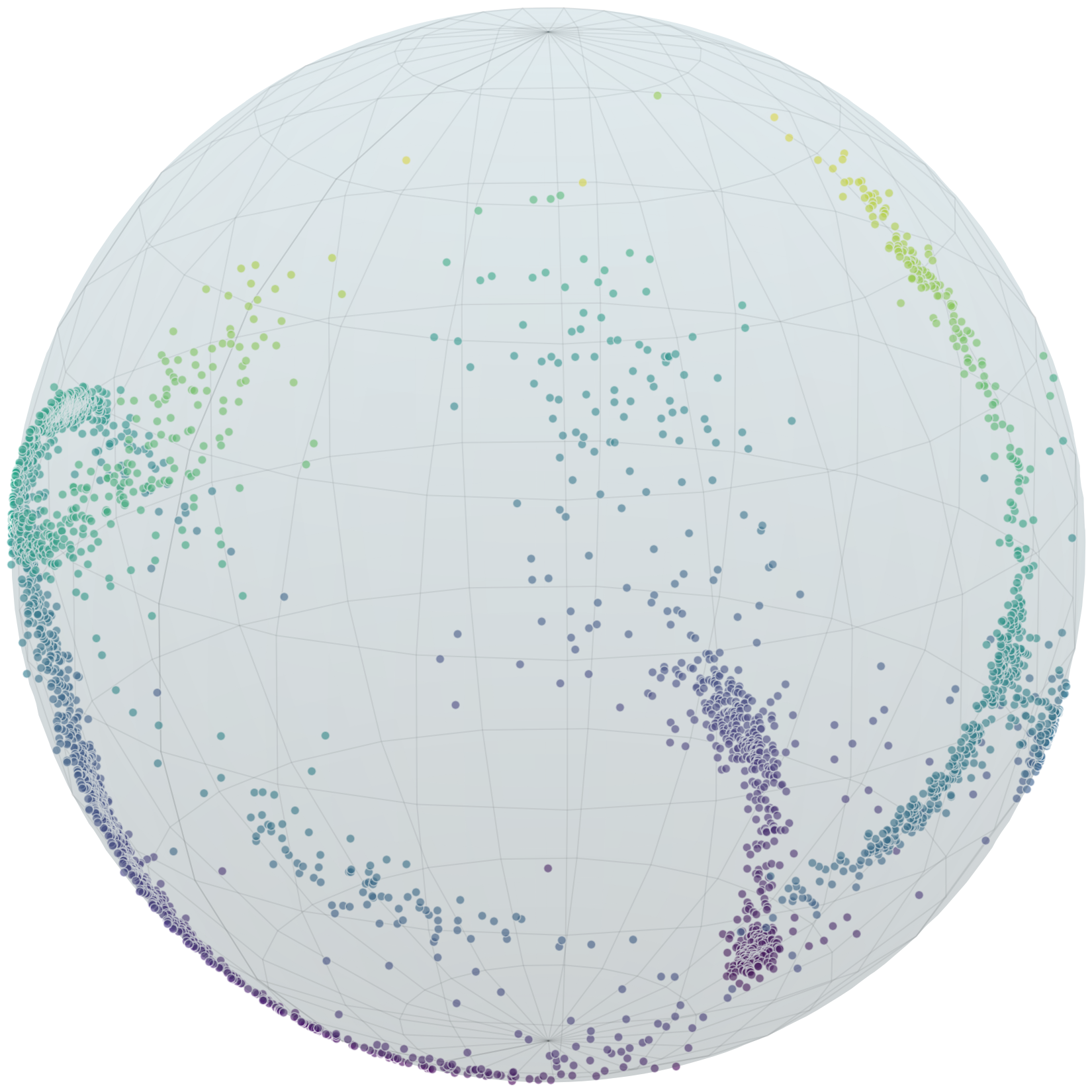}
        \caption{MAD}
    \end{subfigure}
    \hfill
    \begin{subfigure}[b]{0.14\textwidth}
        \includegraphics[width=\textwidth]{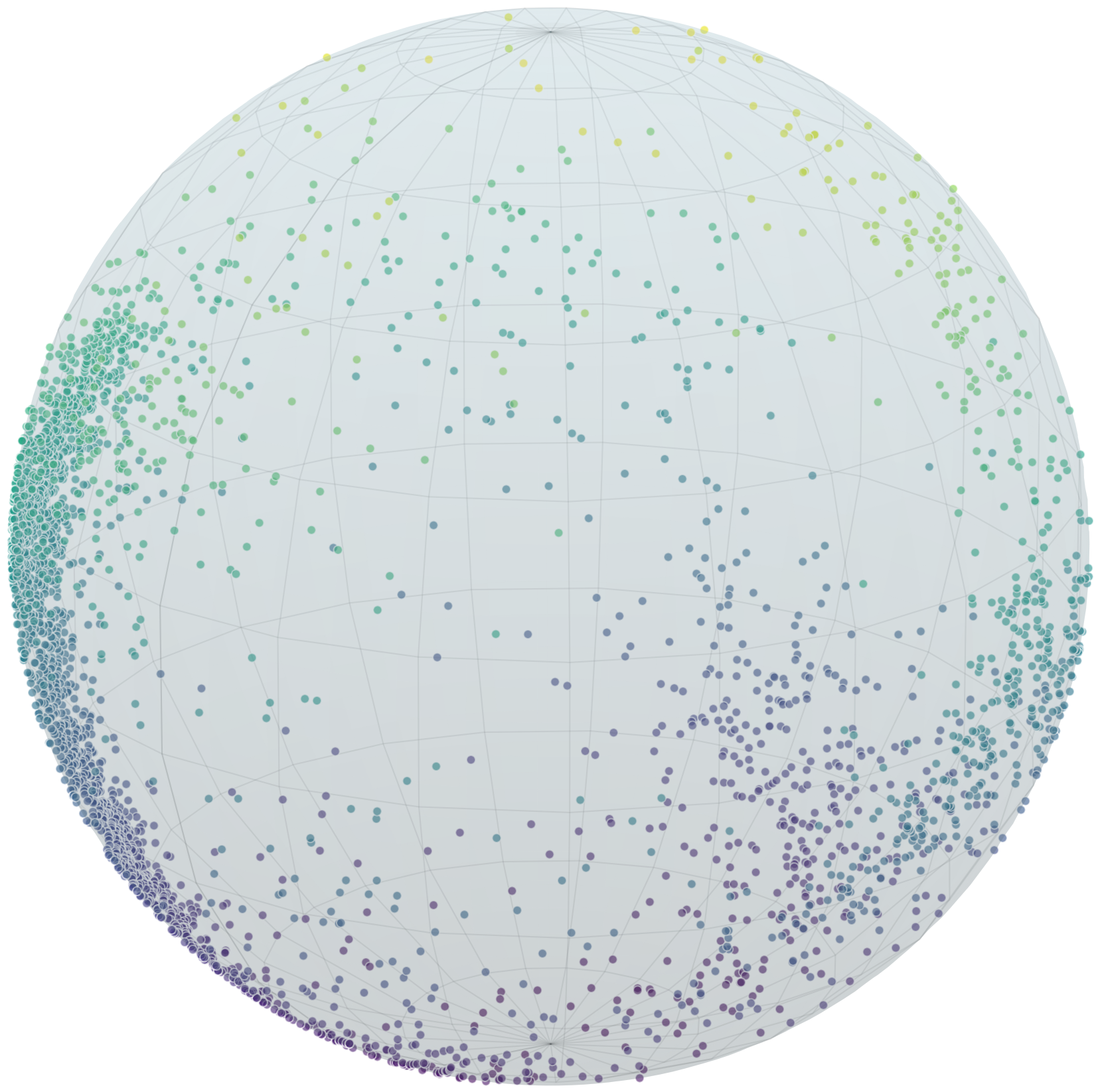}
        \caption{RSGM-ambient}
    \end{subfigure}
    %
    %
    \centering
    \begin{subfigure}[b]{0.14\textwidth}
        \includegraphics[width=\textwidth]{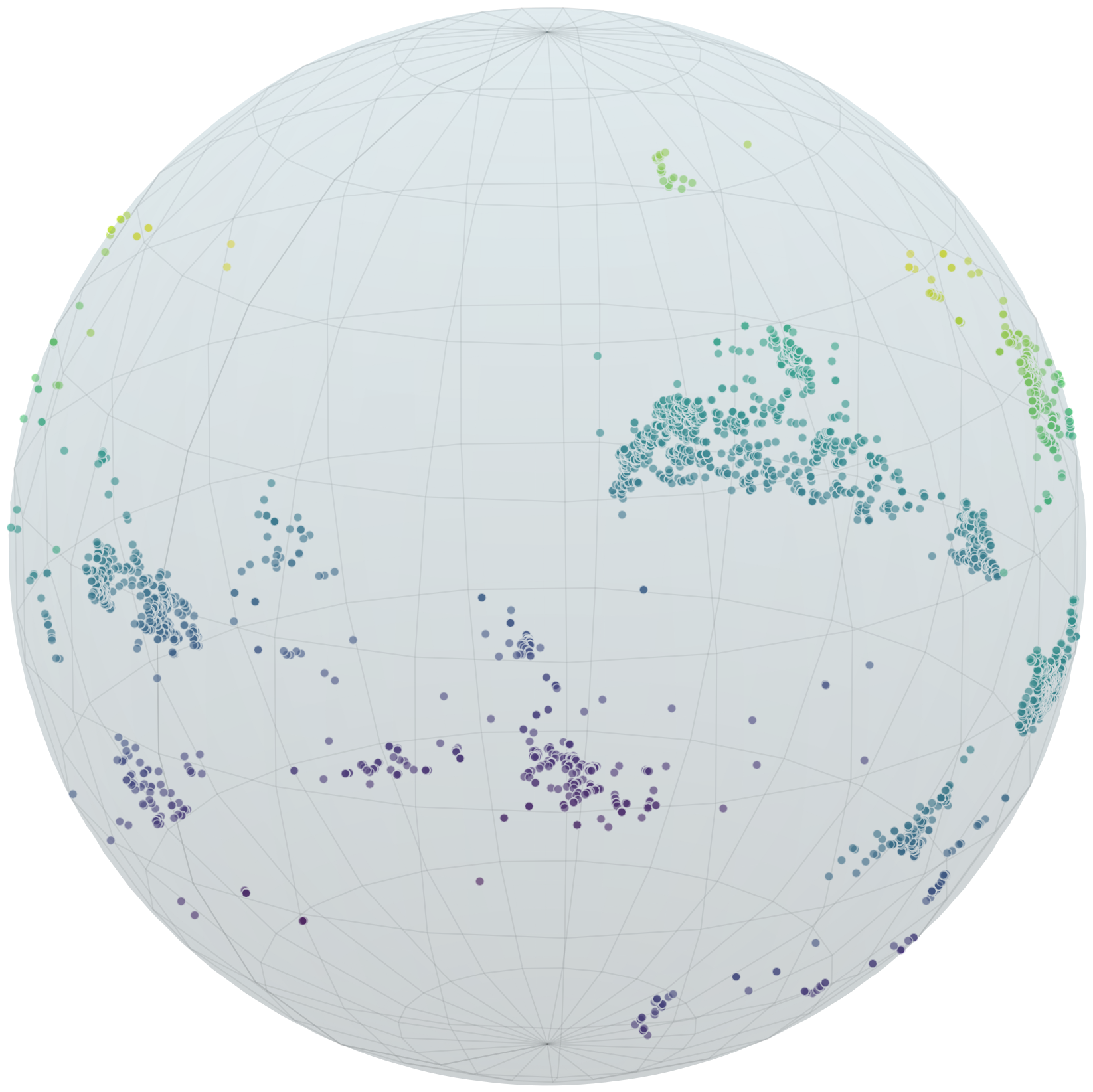}
        \caption{Fire data}
    \end{subfigure}
    \hfill
    \begin{subfigure}[b]{0.14\textwidth}
        \includegraphics[width=\textwidth]{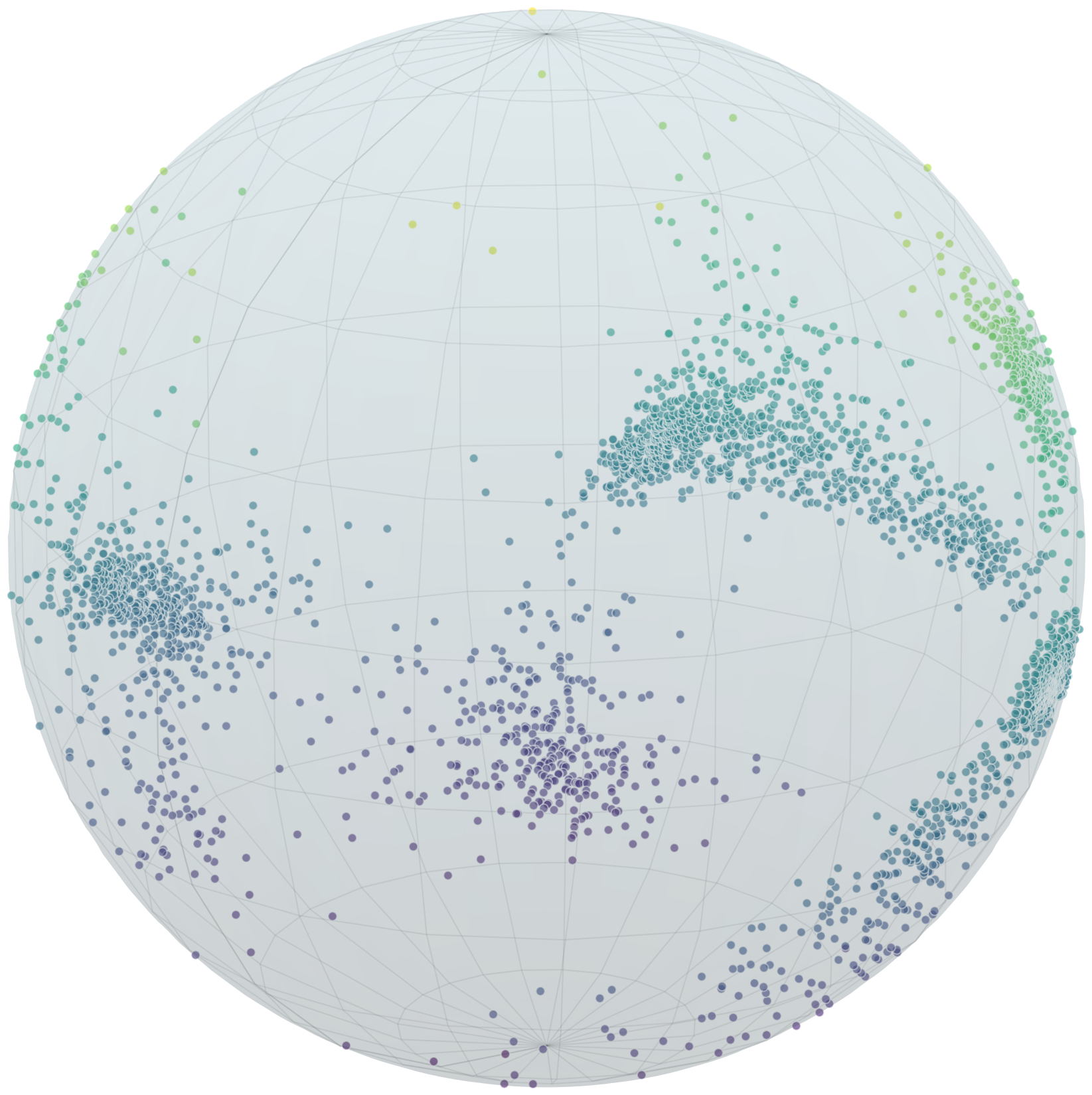}
        \caption{MAD}
    \end{subfigure}
    \hfill
    \begin{subfigure}[b]{0.14\textwidth}
        \includegraphics[width=\textwidth]{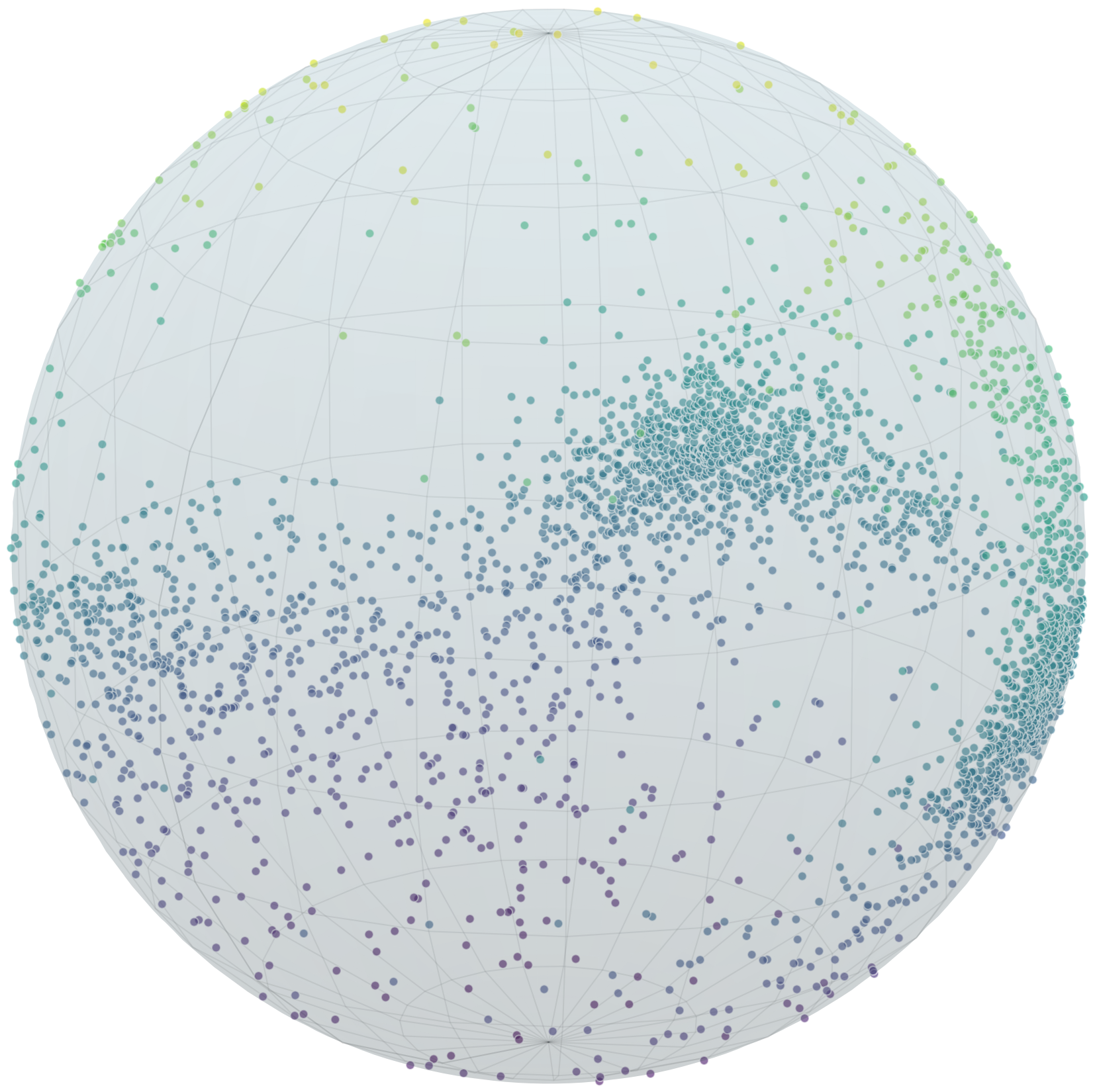}
        \caption{RSGM div-free}
    \end{subfigure}

        \centering
    \begin{subfigure}[b]{0.14\textwidth}
        \includegraphics[width=\textwidth]{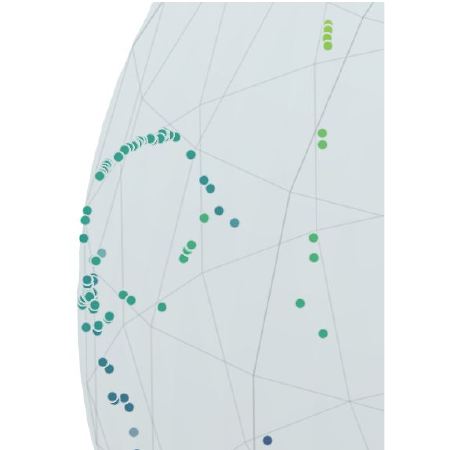}
        \caption{Volcano data} \label{fig:earth_closeup_data}
    \end{subfigure}
    \hfill
    \begin{subfigure}[b]{0.14\textwidth}
        \includegraphics[width=\textwidth]{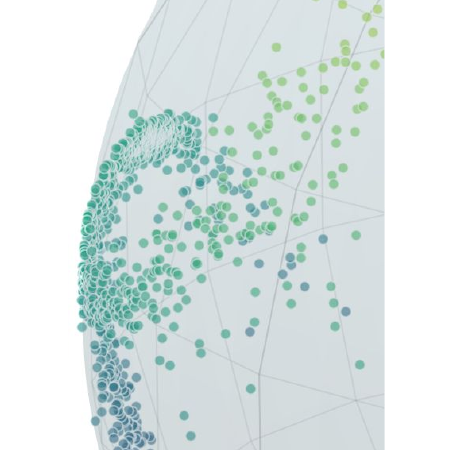}
        \caption{MAD} \label{fig:earth_closeup_ours}
    \end{subfigure}
    \hfill
    \begin{subfigure}[b]{0.14\textwidth}
        \includegraphics[width=\textwidth]{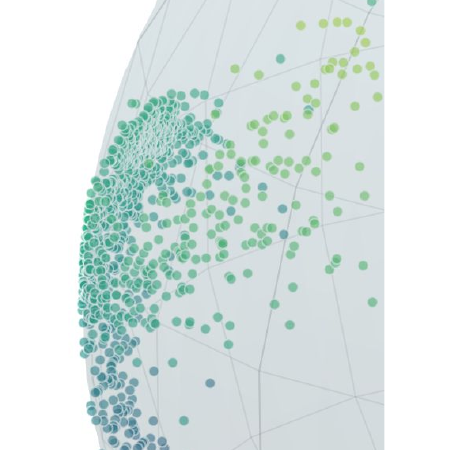}
        \caption{DSM} \label{fig:earth_closeup_dsm}
    \end{subfigure}
    \caption{Earth data (left panel) vs. generated samples from different methods.The colour map, defined by the polar axis, is left for visualisation purposes.}
    \label{fig:earth}
\end{figure}

\subsection{Rotations in $\SO(3)$}
We next continue to a more complicated task -- generating rotation matrices in $\SO(3)$, represented as quaternions in $S^3$. We test Gaussian mixtures in $\SO(3)$ as experimented in \cite{de2022riemannian} with increasing complexity from $K=16$ components (easiest) to $K=64$ components (toughest). As shown in Table \ref{tab:so3}, for $K=16$ RSGM outperforms all methods. However its MMD increases with distribution complexity until all methods are on par at $K=64$. Notably, MAD exhibits the fastest convergence (Figure \ref{fig:so3_mmd_vs_time}) at effectively no added sampling cost compared to DSM (Table \ref{tab:so3}). This suggests that incorporating a base score improves optimisation without sacrificing the simplicity of ambient-space sampling. Figure \ref{fig:so3_plots} shows qualitative differences between the different methods, with slightly better separation between the different components using MAD compared to DSM, and similar separation between MAD and RSGM.

\begin{figure}[ht]
  \vskip 0.2in
  \begin{center}
    \begin{minipage}[c]{0.35\columnwidth}
        \begin{subfigure}[t]{1\linewidth}
            \includegraphics[width=\linewidth]{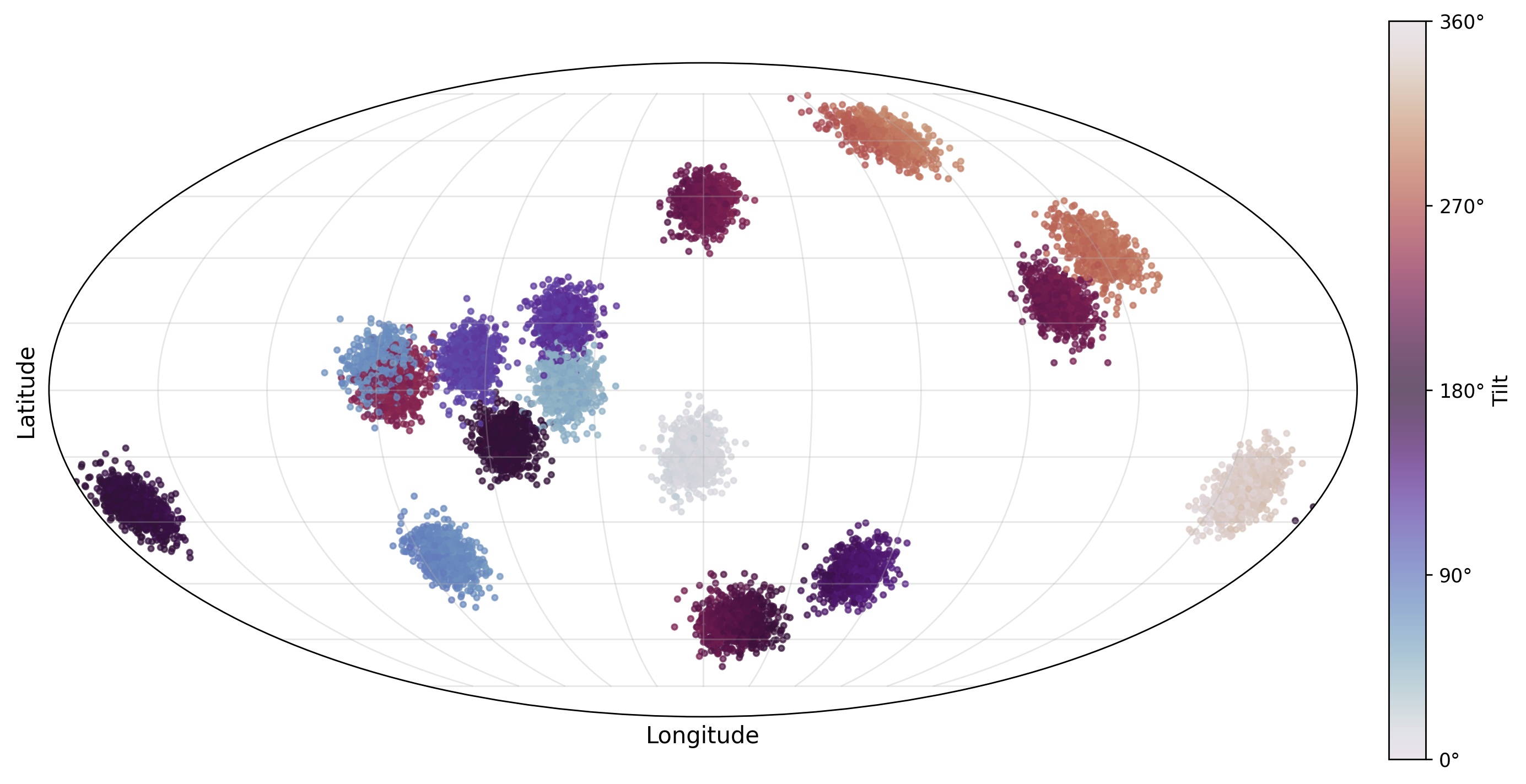}
            \caption{Data}
        \end{subfigure}\hfill
    \end{minipage}
    \hfill
    \begin{minipage}[c]{0.63\columnwidth}
        \centering
        \begin{subfigure}[t]{0.5\linewidth}
            \includegraphics[width=\linewidth]{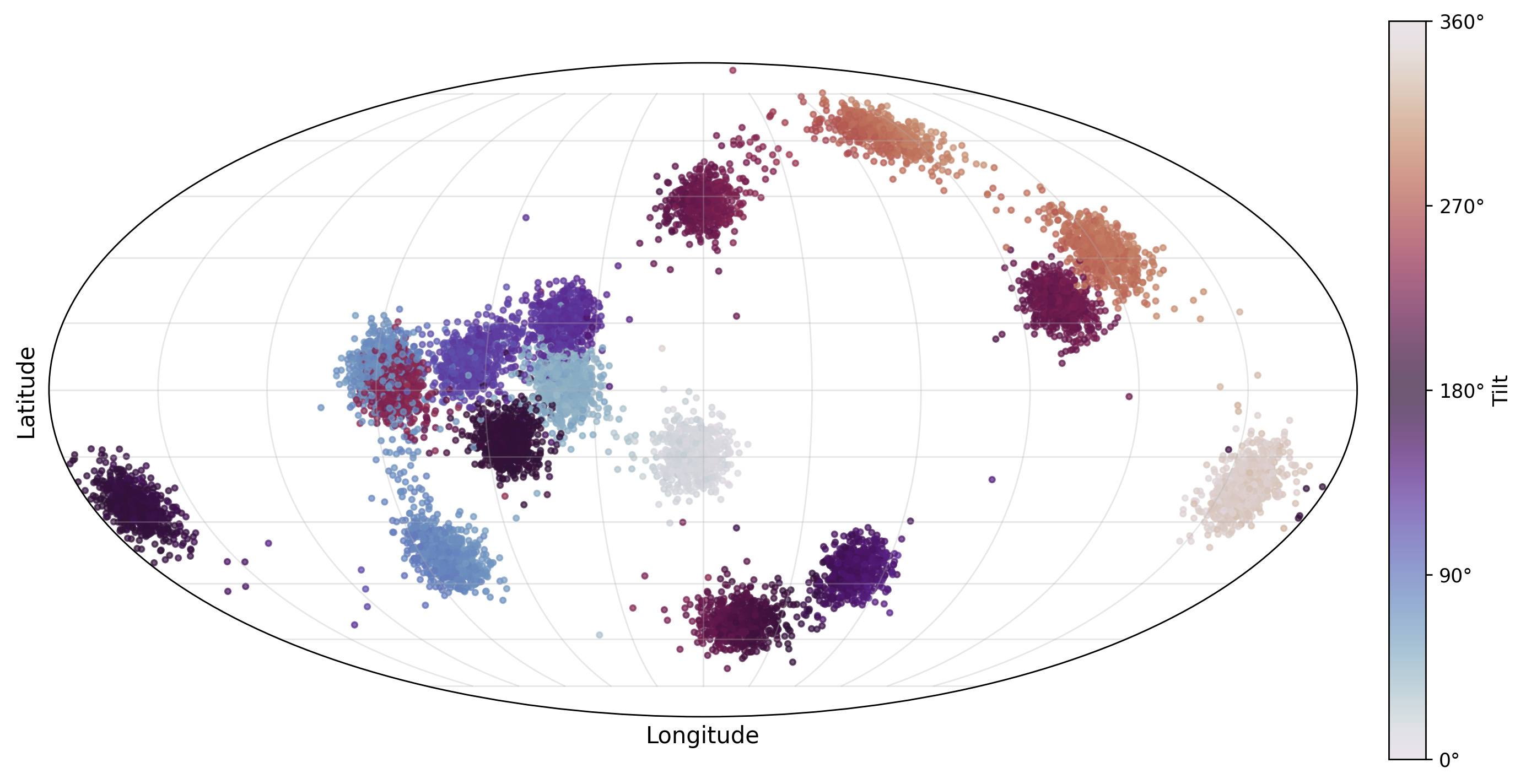}
            \caption{MAD}
        \end{subfigure}\hfill
        \begin{subfigure}[t]{0.48\linewidth}
            \includegraphics[width=\linewidth]{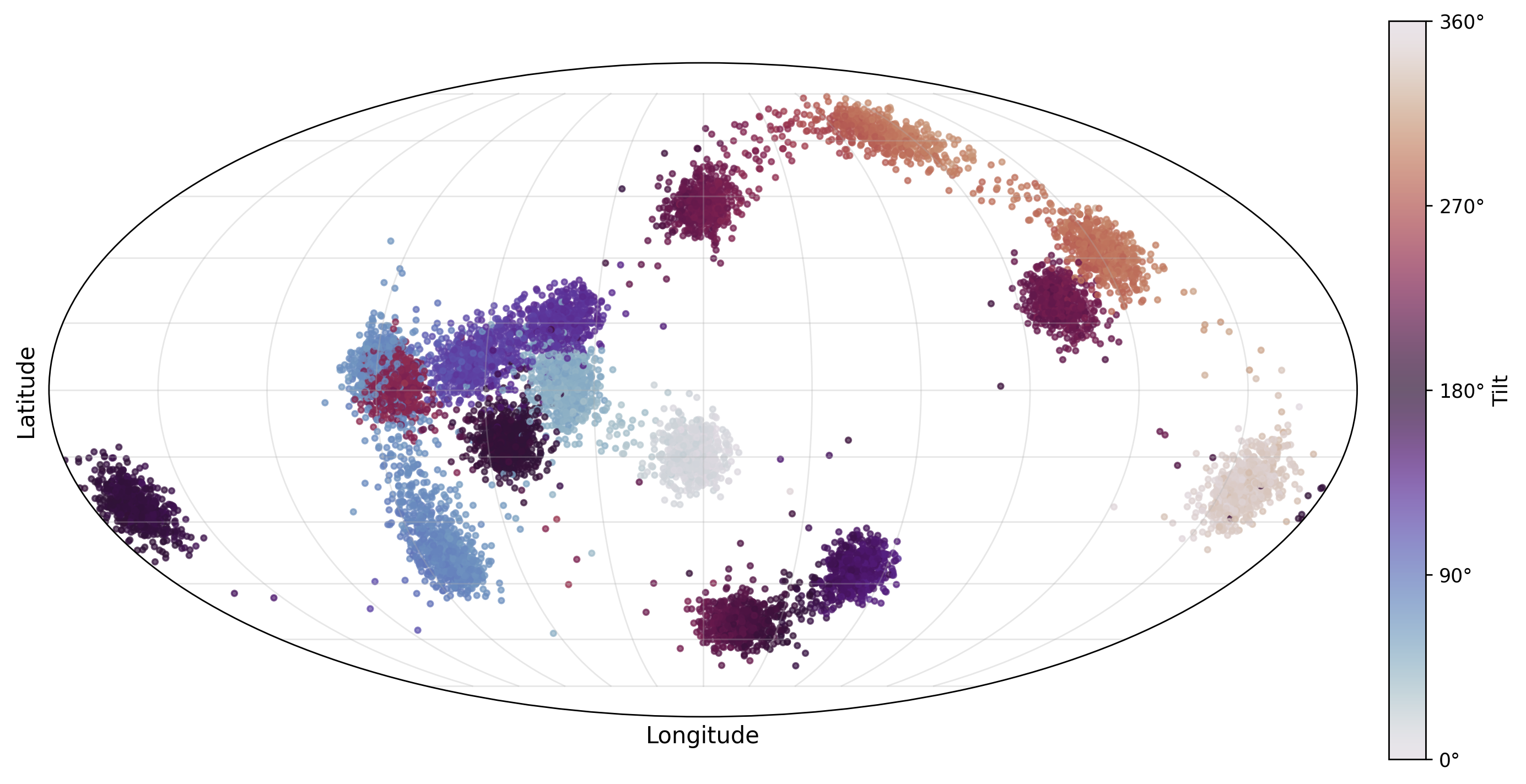}
            \caption{DSM}
        \end{subfigure}
    
        \medskip
    
        \begin{subfigure}[t]{0.48\linewidth}
            \includegraphics[width=\linewidth]{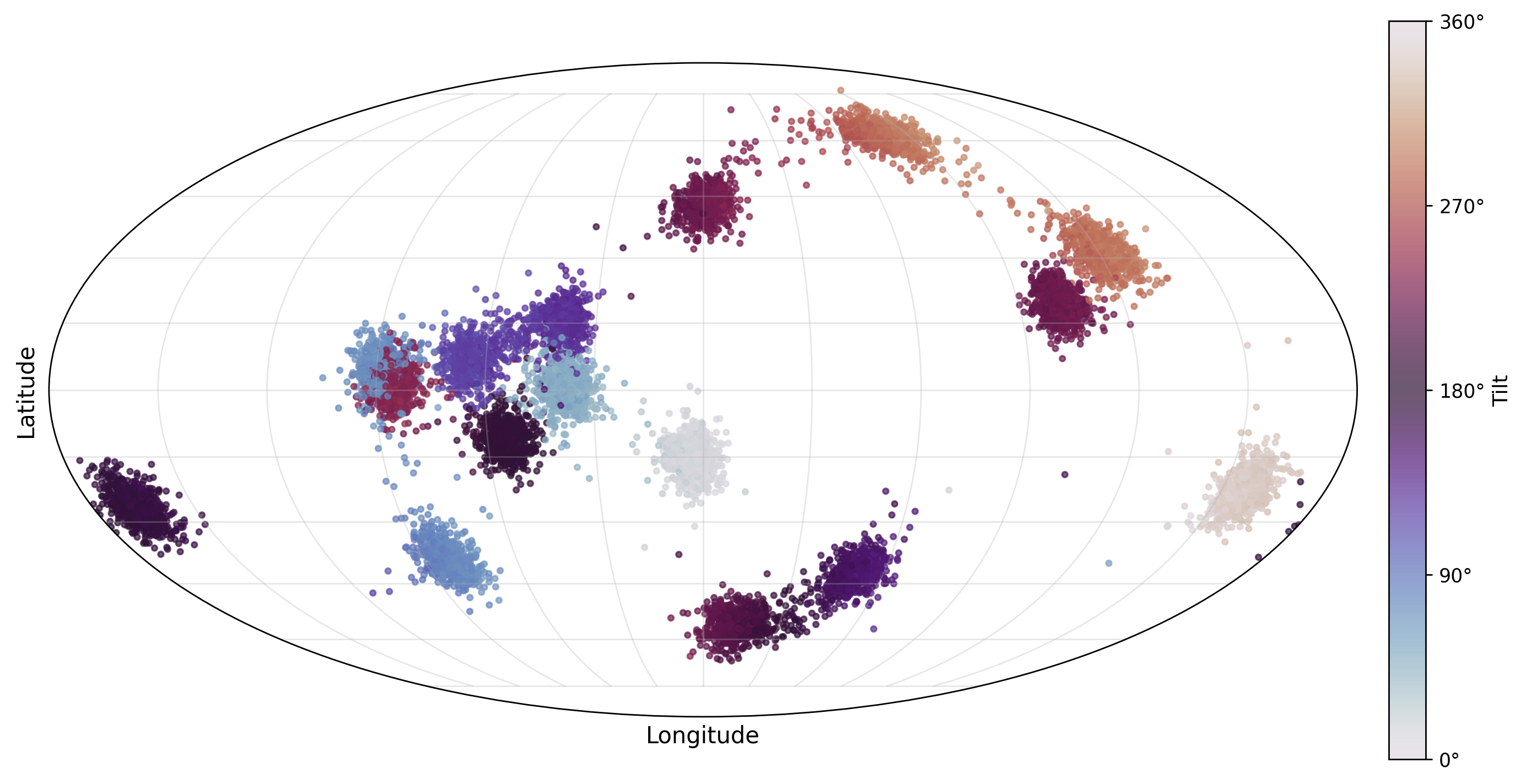}
            \caption{RSGM}
        \end{subfigure}\hfill
        \begin{subfigure}[t]{0.48\linewidth}
            \includegraphics[width=\linewidth]{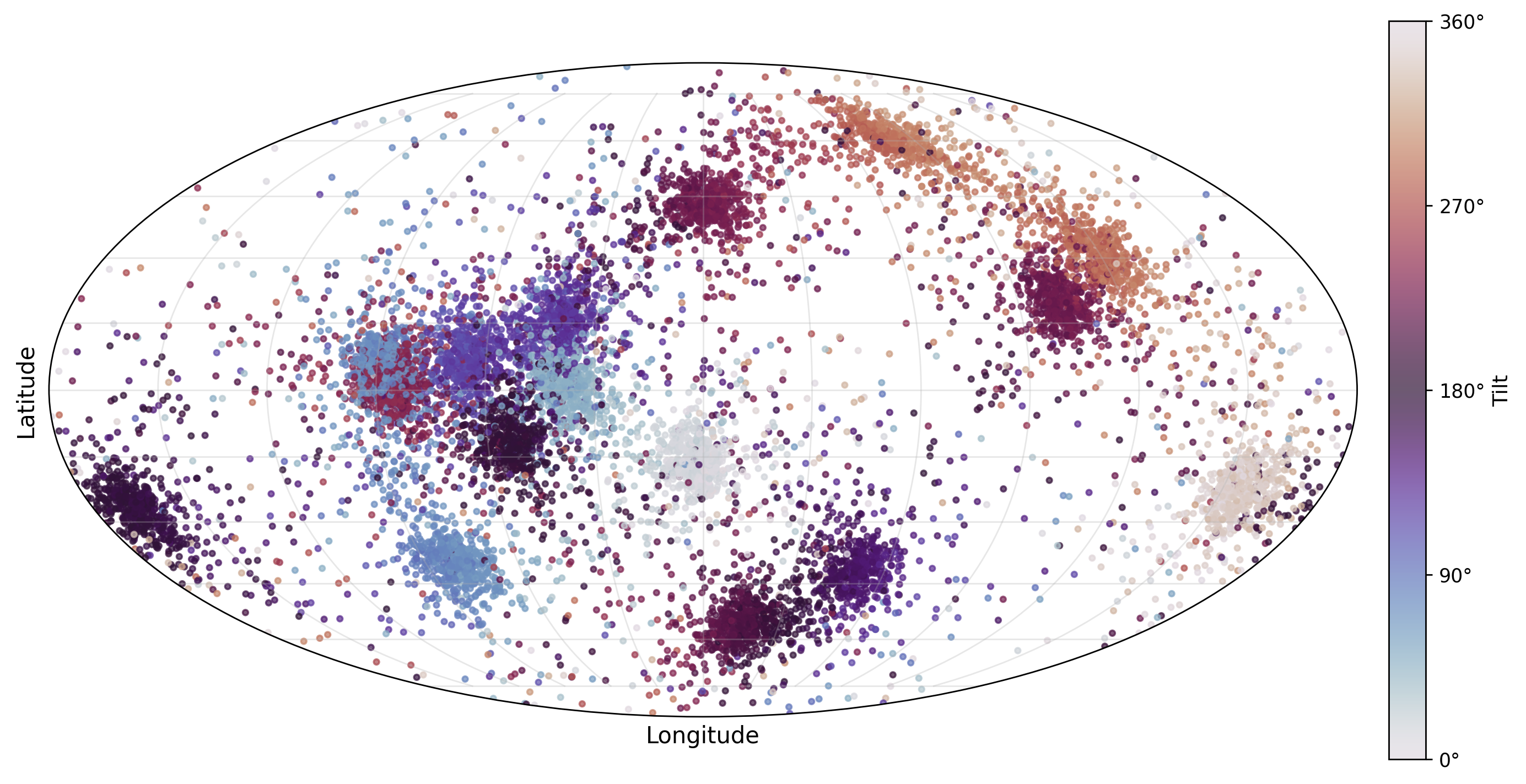}
            \caption{FFF}
        \end{subfigure}
    \end{minipage}
    \caption{$\SO(3)$ results for $K=16$ in the form of a Mollweide projection for different methods and data. All figures are from seed $0$. Points are coloured according to angle, not component.}
    \label{fig:so3_plots}
      \end{center}
\end{figure}

\begin{figure}[ht]
    \centering
        \includegraphics[width=\columnwidth]{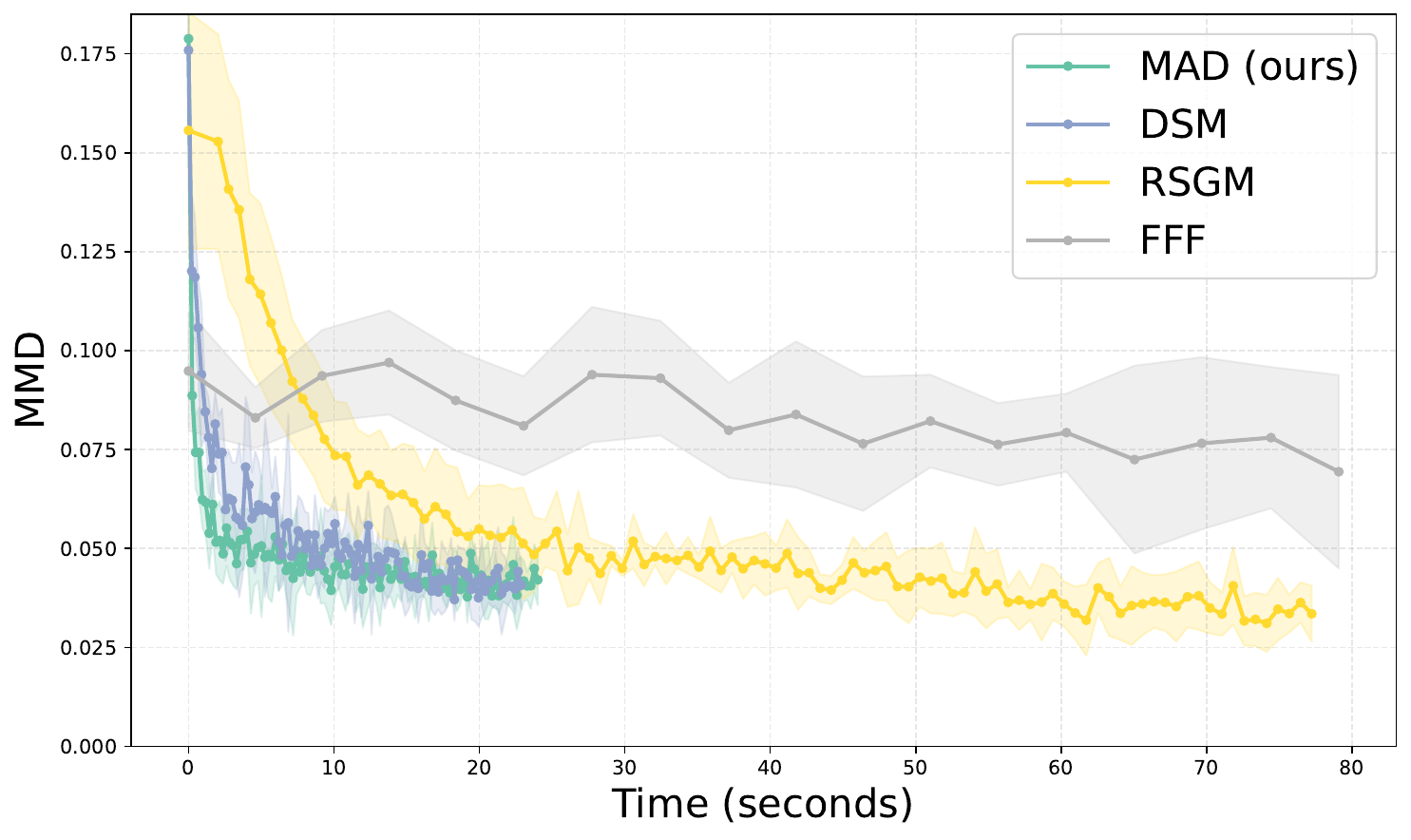}
    \caption{MMD vs cumulative training-step time. Shaded areas are ± std across $5$ runs. MMD is computed against the training set for $K=64$ using 1,000 samples for efficiency (the table uses 5,000). Note that for better clarity of comparison between all methods, the time shown is up to 80, but FFF ends after 400. } \label{fig:so3_mmd_vs_time}
\end{figure}

\textbf{SymSol: $\SO(3)$ with symmetries.} 
We study conditional generative modelling of 3D rotations for objects with discrete rotational symmetries. Concretely, given a single 2D image $c$ of a rigid object, the goal is to learn a conditional distribution over its pose,
$p(R|c)$ with $R \in \SO(3)$, and to generate samples from this posterior. Our experiments use the \textsc{SymSol I} dataset \citep{murphy2022implicitpdfnonparametricrepresentationprobability}, consisting of 2D images of simple symmetric solids (e.g.\ cylinders, cubes, icosahedra) paired with ground-truth rotations (see Figure \ref{fig:symsol}). Since many of these shapes admit a non-trivial symmetry group $G \subset \SO(3)$, e.g. a 2D image of a cylinder can be represented by two distinct rotations, the conditional pose distribution is generally non-identifiable for symmetric solids. Thus, we resort to the quotient-space canonicalisation outlined in \ref{sec:rotations}. Further details on the design of the conditional denoiser and training are in Appendix \ref{app:symsol}. Since \textsc{SymSol I} provides ground truth rotations, we compute the mean \textit{spread}\footnote{Spread evaluates the minimum expected average angular deviation $\mathbb{E}_{R\sim p(R|I)}[\min_{R'\in \{R_{gt}g | g \in G\}]} d(R, R')]$, where $d(R,R')$ denotes the geodesic distance.} as suggested in \citep{liu2023delvingdiscretenormalizingflows, murphy2022implicitpdfnonparametricrepresentationprobability} as our key metric. 

\begin{figure}[ht]
    \centering
    \includegraphics[width=\columnwidth]{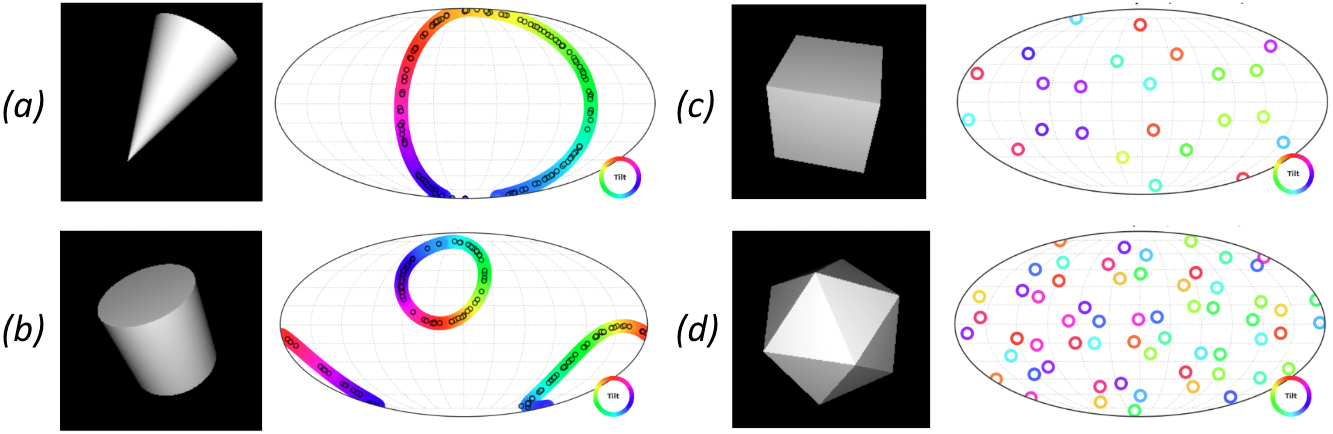}
    \caption{2D image snapshots of a randomly rotated cone, cylinder, cube, icosahedron (a-d) and Mollewide projections of the corresponding pose distributions in $\SO(3)$. Black circles in (a-b) denote samples from MAD and are excluded from (c-d) for visual clarity.}
    \label{fig:symsol}
\end{figure}

For all shapes, MAD demonstrates improved performance with respect to DSM, once again demonstrating faster convergence (Figure \ref{fig:symsol_training}). MAD also shows superior guarantees in reaching the manifold, with consistently low manifold deviation across samples. Notably, for more complex shapes with higher-order symmetries such as the cube ($|G|=24$) and the icosahedra ($|G|=60$), MAD outperforms DSM, suggesting MAD spends more optimization efforts learning the otherwise rather complex distribution. Finally, we demonstrate MAD is on-par with prior methods in Table \ref{tab:symsol}, highlighting that our score decomposition and quotient-space canonicalisation offers an efficient and simple solution to tackling complex distributions in $\SO(3)$.

\begin{figure*}[ht]
    \centering
    \includegraphics[width=\textwidth]{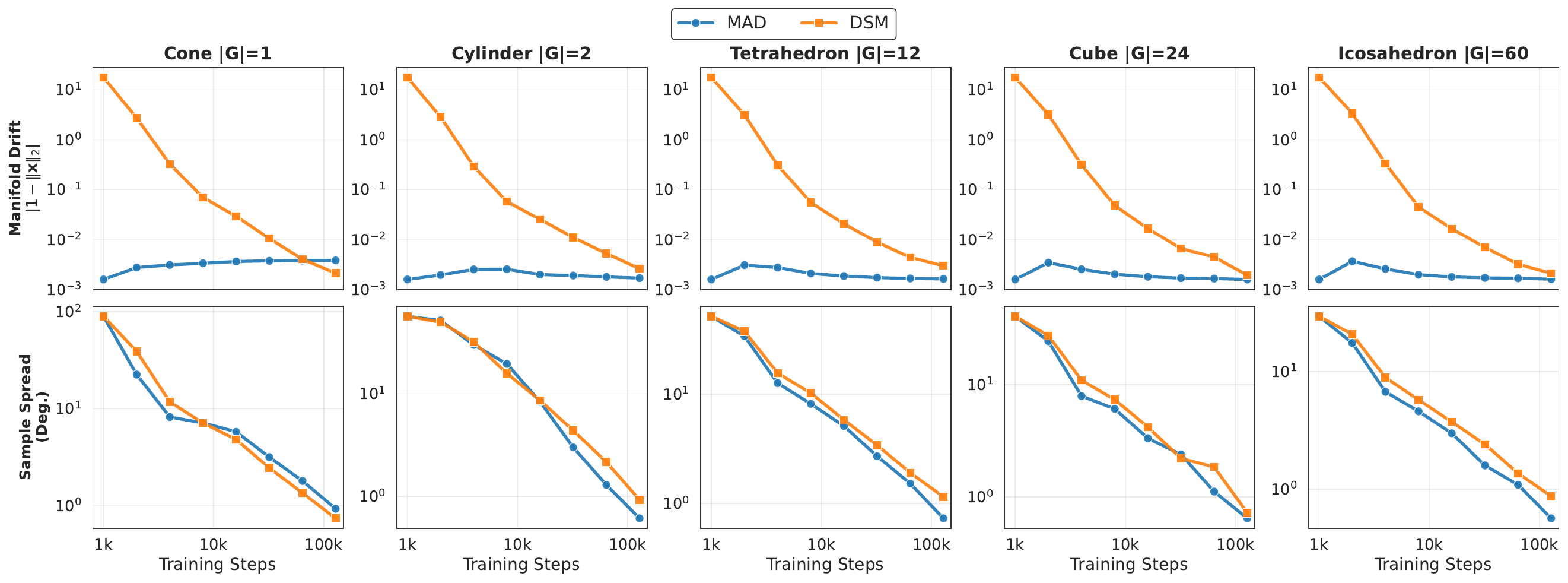}
    \caption{Top Row: Manifold Drift ($|1 - \|x\|_2|$). MAD (blue) maintains consistently lower drift, indicating predictions stay closer to valid rotations. DSM (orange) exhibits higher drift, indicative of 'ghost rotation' averaging. Bottom Row: Sample Spread (lower is better).}
    \label{fig:symsol_training}
\end{figure*}

\begin{table*}[t]
    \caption{Spread estimation on \textsc{SymSol I}. Values are in degrees ($\downarrow$). Baseline results are taken from \citet{liu2023delvingdiscretenormalizingflows}.}
    \label{tab:symsol}
    \centering
    \begin{small}
    \begin{sc}
    \begin{tabular}{lcccccc}
        \toprule
        Method & Cone & Cyl. & Tet. & Cube & Ico. & Avg. \\
        \midrule
        \citet{deng2020deepbinghamnetworksdealing} & 10.1 & 15.2 & 16.7 & 40.7 & 29.5 & 22.4 \\
        IPDF \citep{murphy2022implicitpdfnonparametricrepresentationprobability} & 1.4 & 1.4 & 4.6 & 4.0 & 8.4 & 4.0 \\
        DNF \cite{liu2023delvingdiscretenormalizingflows} & \textbf{0.5} & \textbf{0.5} & \textbf{0.6} & \textbf{0.6} & 1.1 & \textbf{0.7} \\
        \midrule
        DSM + Quotient & 0.7 & 0.9 & 1.2 & 0.7 & 0.9 & 0.9 \\
        MAD + Quotient (ours) & 0.9 & 0.6 & 0.7 & \textbf{0.6} & \textbf{0.6} & \textbf{0.7} \\
        \bottomrule
    \end{tabular}
    \end{sc}
    \end{small}
\end{table*}

\subsection{Discrete Data}
Finally, an increasingly important but complicated task is learning to generate from discrete distributions. There has been great interest recently in doing so with diffusion models, mainly in the context of large-language models (e.g. \cite{shi2024simplified}). 
Discrete distributions also represent an extreme case of low-dimensional structure, where the support consists of isolated points. This setting directly tests how well our method can incorporate support information.
We use a distribution on a discrete set of points on the unit circle shown in Figure \ref{fig:discrete}. Note that we do not project the samples onto the support, so that the proximity of the generated samples to the support is visible. As a first simple task, we assume a uniform distribution (Figure \ref{fig:discrete_data_uniform}). Not surprisingly, our method is able to learn this distribution easily since the difference between the base score and the real score is zero, leaving only the base score to take effect. In comparison, DSM generates out-of-distribution samples between each of the support points. Next we complicate the target distribution to a skewed distribution (Figure \ref{fig:discrete_data_skewed} and Appendix \ref{appx:data_discrete}). Here the difference becomes even more pronounced, with MAD strongly resembling the target distribution, while  DSM has many out-of-distribution samples. The ability of MAD to generate samples close to the support may be expected as a consequence of Theorem \ref{thm:o(1)_discrete}, which shows that as $\sigma_t$ decreases, the remaining component of the score will become approximately $0$. This is also relates to the findings from \cite{li2025scores} whereby the error between the learned and true score functions must be within $o(1)$ in order to recover the true distribution. 

\begin{figure}[ht] 
    \centering
        \begin{subfigure}[b]{0.14\textwidth}
        \includegraphics[width=\textwidth]{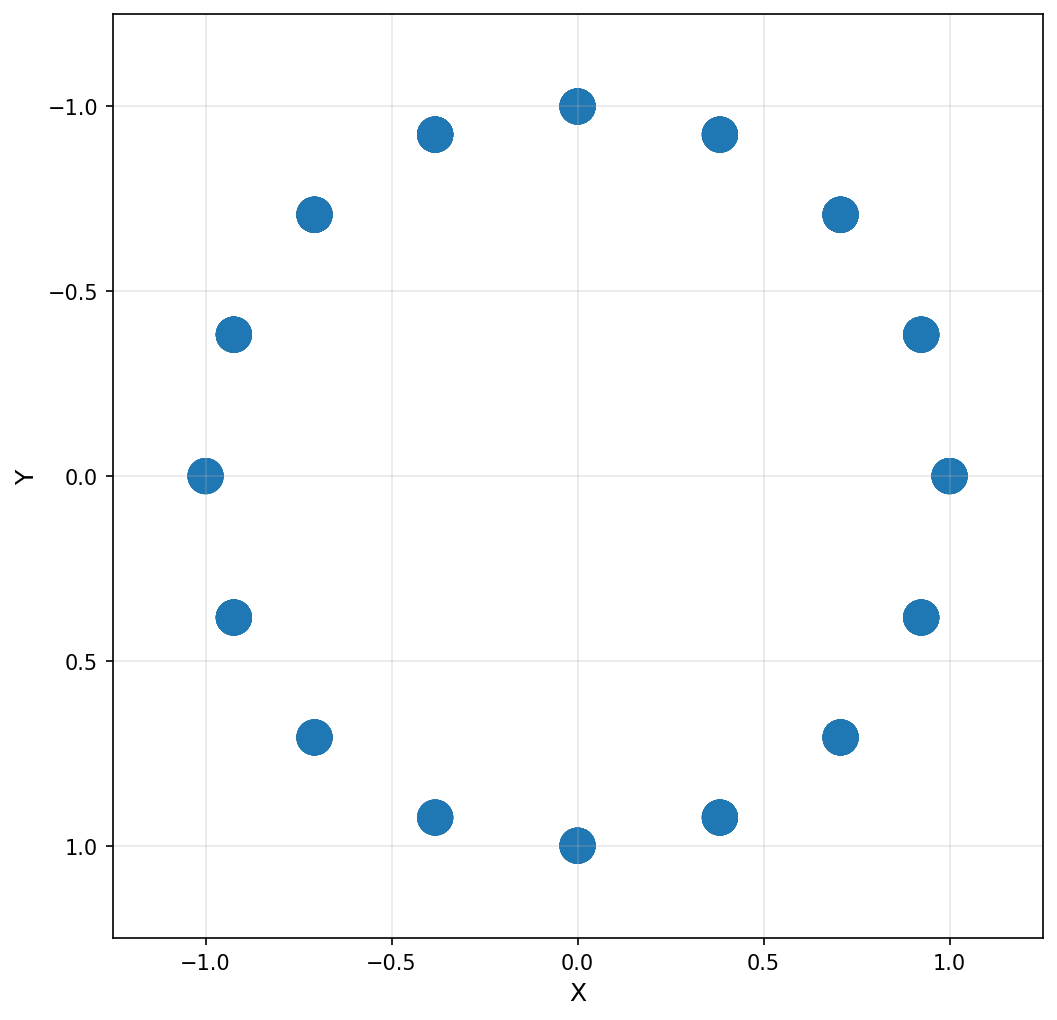}
        \caption{Data uniform} \label{fig:discrete_data_uniform}
    \end{subfigure}
    \hfill
    \begin{subfigure}[b]{0.14\textwidth}
        \includegraphics[width=\textwidth]{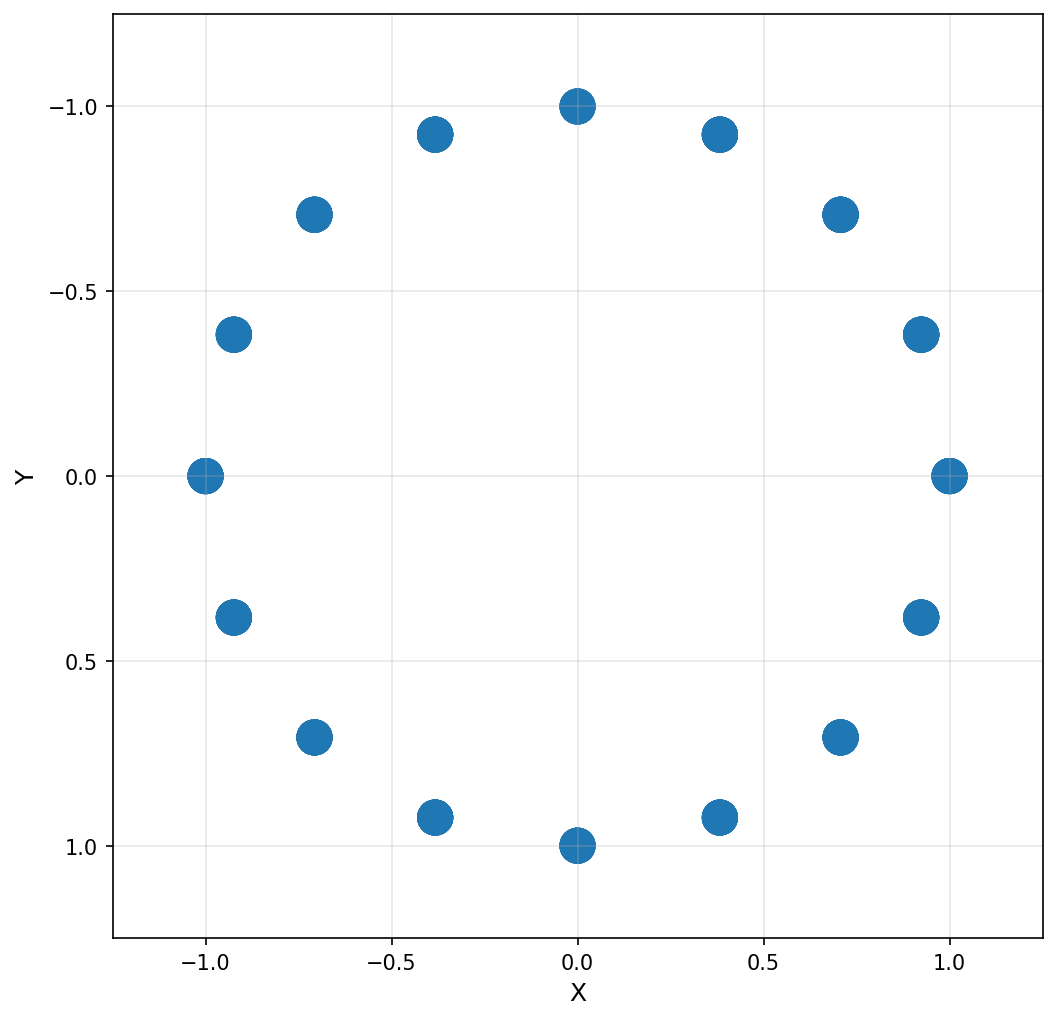}
        \caption{MAD uniform} \label{fig:discrete_mad_uniform}
    \end{subfigure}
    \hfill
    \begin{subfigure}[b]{0.14\textwidth}
        \includegraphics[width=\textwidth]{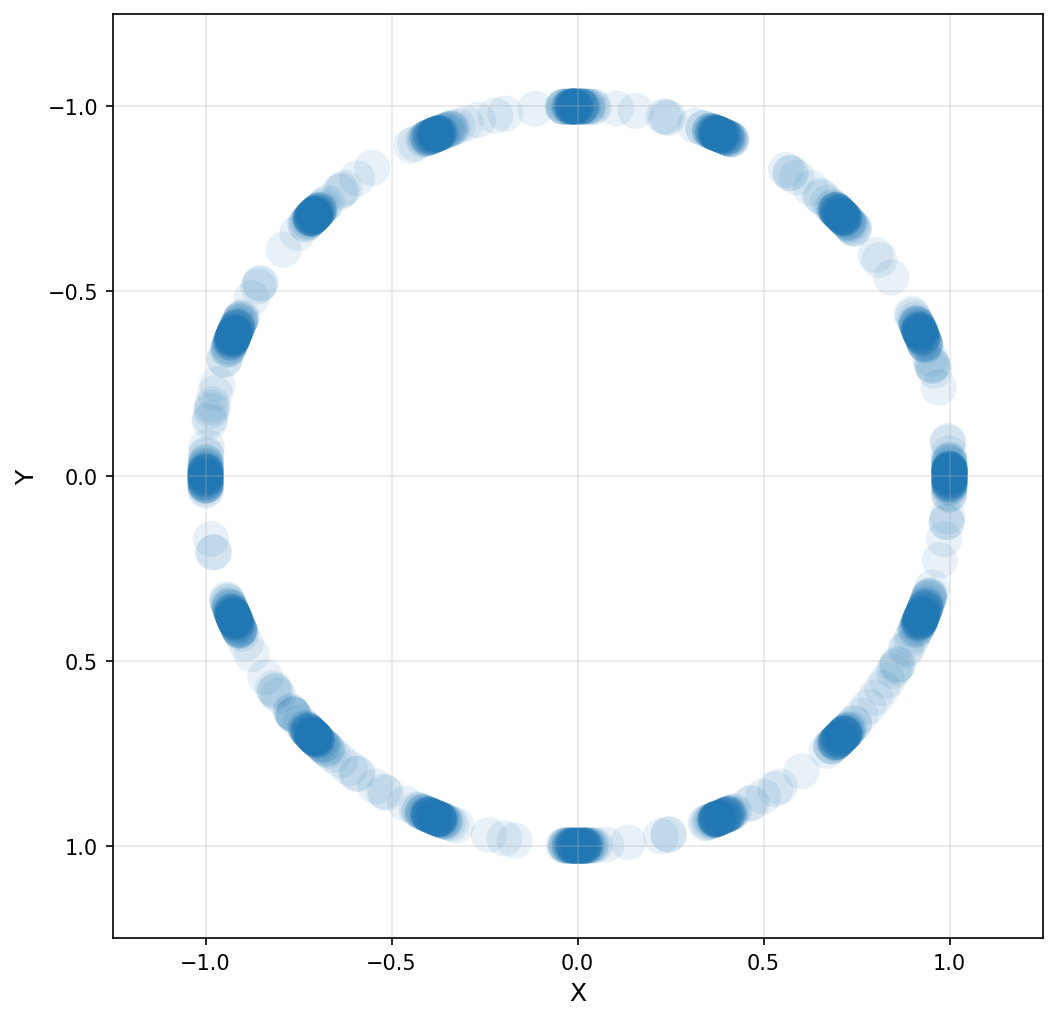}
        \caption{DSM uniform} \label{fig:discrete_dsm_uniform}
    \end{subfigure}
    \centering
    \begin{subfigure}[b]{0.14\textwidth}
        \includegraphics[width=\textwidth]{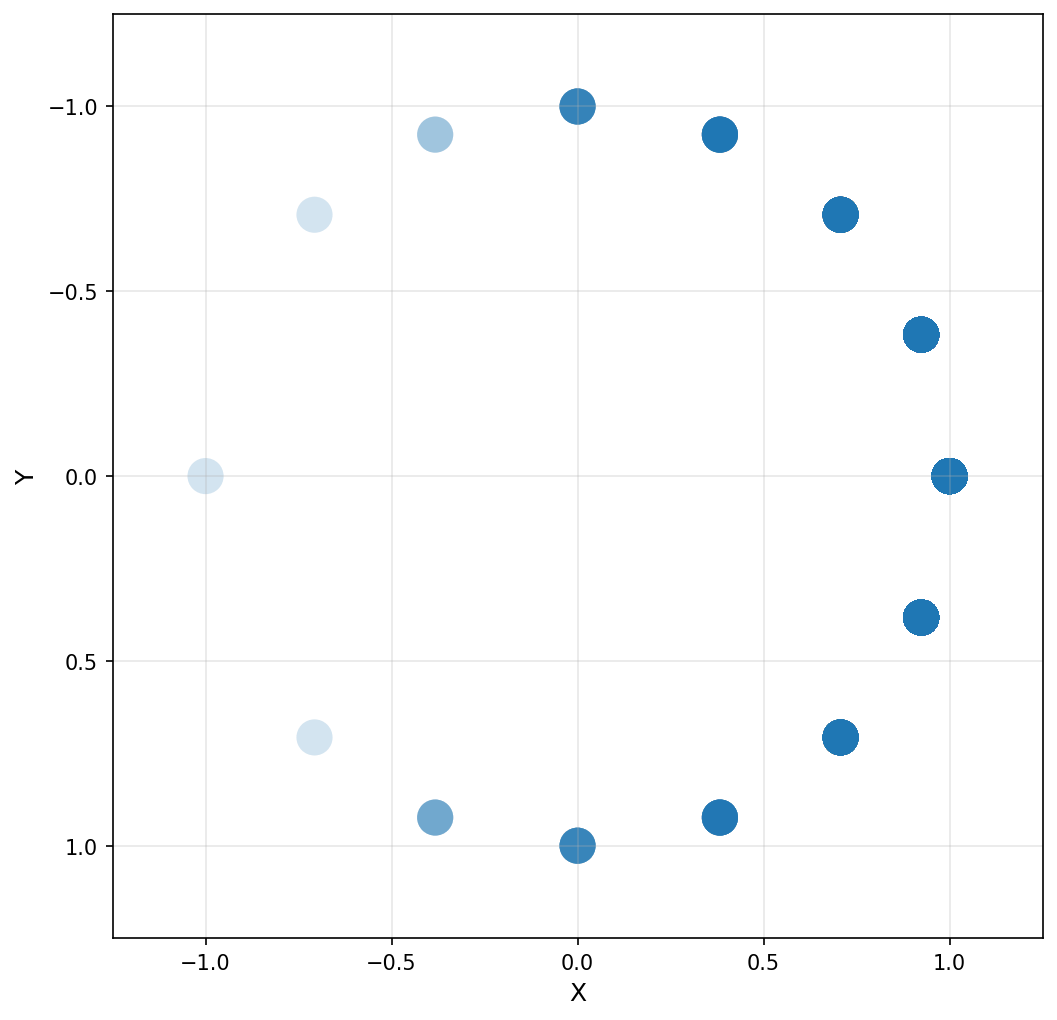}
        \caption{Data skewed} \label{fig:discrete_data_skewed}
    \end{subfigure}
    \hfill
    \begin{subfigure}[b]{0.14\textwidth}
        \includegraphics[width=\textwidth]{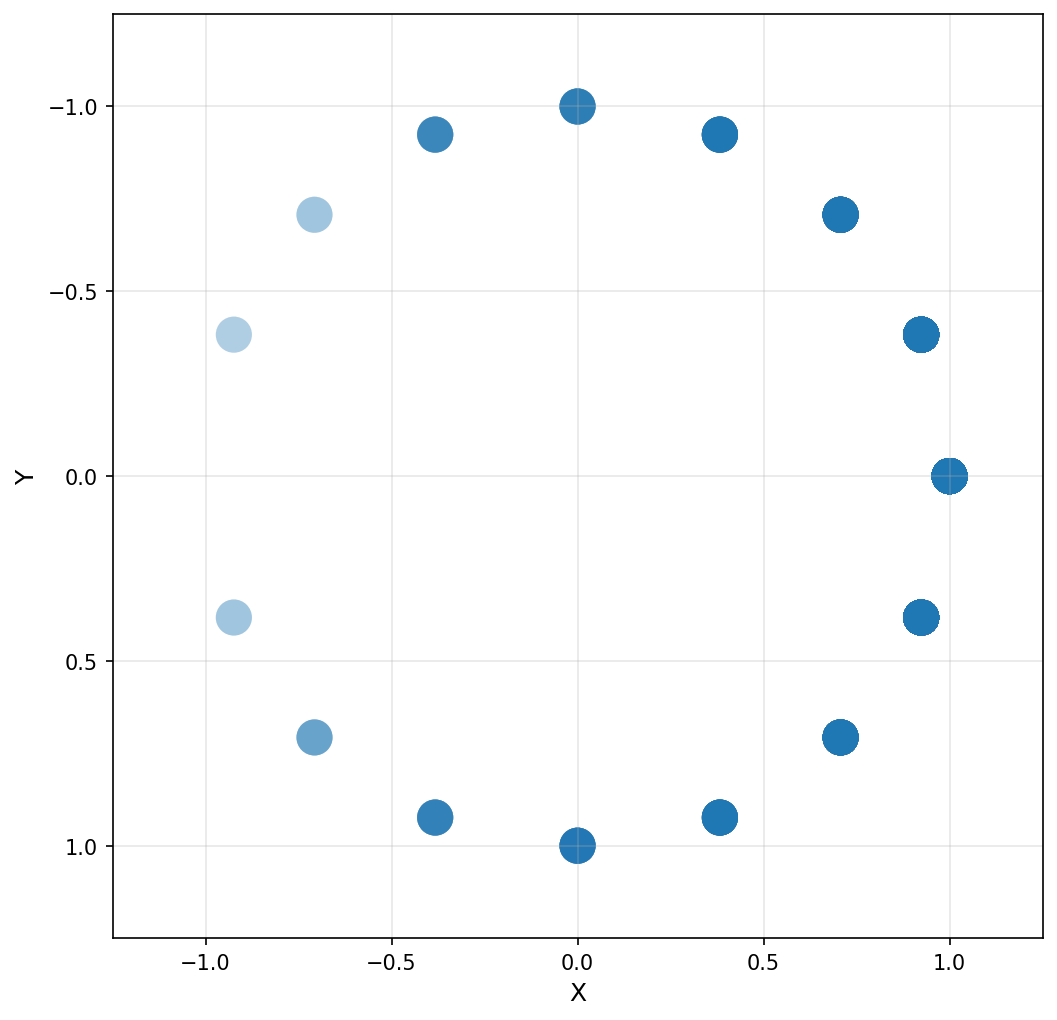}
        \caption{MAD skewed} \label{fig:discrete_mad_skewed}
    \end{subfigure}
    \hfill
    \begin{subfigure}[b]{0.14\textwidth}
        \includegraphics[width=\textwidth]{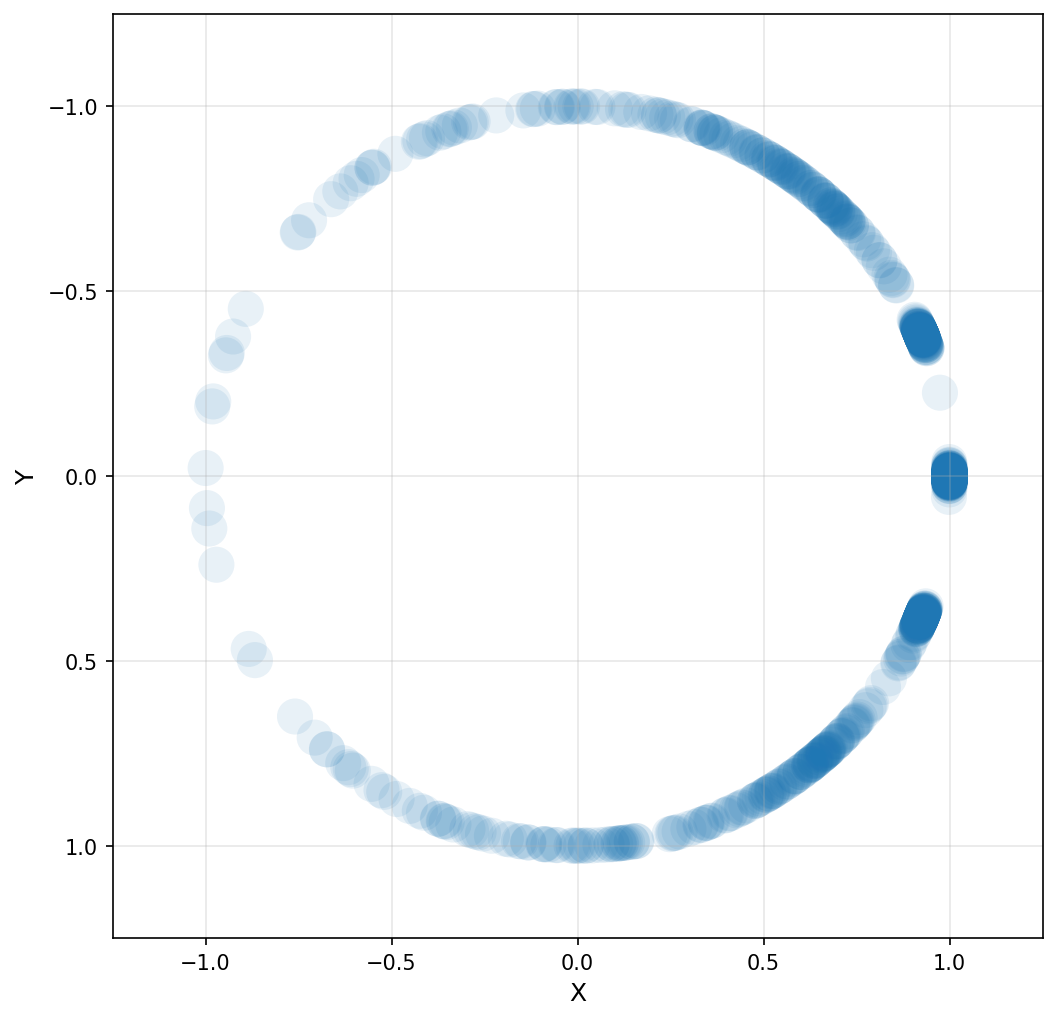}
        \caption{DSM skewed} \label{fig:discrete_dsm_skewed}
    \end{subfigure}
    \caption{Generating discrete distributions -- uniform (top) and skewed (bottom). Data (left) and generated samples from each method (middle, right). Samples are not projected onto the support. }\label{fig:discrete}
\end{figure}

\section{Related Work}


\subsection{On-Manifold Methods}
On-manifold methods explicitly model the generative process directly on the data manifold, typically by defining ODEs or SDE within the manifold geometry \cite{ falorsi2020neural, lou2020neural, mathieu2020riemannian, rozen2021moser, huang2022riemannian, ben2022matching, de2022riemannian, chen2023flow}. These methods often assume that the manifold is Riemannian, as they require the use of the tangent space. Recent representative approaches include Riemannian score-based generative models (RSGMs \cite{de2022riemannian}), which define stochastic differential equations on Riemannian manifolds and approximate them using geodesic random walks. Although such methods can result in high quality generation, they typically introduce overhead computation, both during training and sampling. Further, these methods require the Riemannian manifold assumption, which, although broadly applicable, limits their use in discrete spaces.

\subsection{Low-Dimensional Euclidean Methods}
Contrary to the methods in the on-manifold category, these methods do not explicitly define a generative process on the manifold like manifold-based SDEs or ODEs, but rather apply a Euclidean generative process to low-dimensional Euclidean spaces and use charts or  other means \cite{gemici2016normalizing, brehmer2020flows, kalatzis2021density} to map the generated samples in each back to the manifold. This often requires multiple mappings, which can complicate optimisation. It is also typically assumed that the manifold is smooth, which may not always be the case. Use of such mappings can also introduce distortions or yield different generated distributions depending on the choice of mapping. Finally, these methods do not immediately apply to discrete cases since the low-dimensional manifold would consist of a single point.

\subsection{Ambient Space Methods}
We include in this class methods that produce samples in the ambient space followed by a projection from the ambient space to the manifold. DSM, when applied to low-dimensional manifolds, can be viewed as such an approach. Specifically, as described by \cite{song2019generative}, when applied to low-dimensional manifolds it is assumed that the target distribution is slightly noised so as to obtain a full-dimensional approximation of it. After learning this approximation, it is necessary to project the samples onto the manifold from the ambient space in order to obtain on-manifold samples. This is computationally fast and simple to deploy, although the approximation in the ambient space needs to be sufficiently accurate and close to the manifold for the projection to produce good results. Since DSM is not aware of the manifold except via the samples provided from the target distribution, it has to first spends some effort on learning the manifold itself, as shown by \cite{li2025scores}. 

Free-Form Flows for manifolds (FFF) \cite{sorrenson2024learning} may also be viewed as ambient space approaches, as they rely on encoder–decoder architectures that generate samples in the ambient space followed by a projection onto the manifold. While FFF enables fast sampling, it too assumes a Riemannian manifold structure and introduces multiple auxiliary loss terms, one of which is needed to encourage proximity between ambient and projected samples. These additional objectives increase training complexity. 

MAD falls within the ambient-space class but differs in that it incorporates prior knowledge of the manifold through a known base score. This information pushes samples towards the manifold during sampling and  may reduce the burden of learning the manifold compared to other ambient space approaches, without introducing additional geometric assumptions or auxiliary training objectives.

\section{Discussion and Limitations}
This work adapts DSM to explicitly account for manifold structure through the use of a base score. Our experiments show that MAD not only improves convergence speed but also captures finer distributional detail. This may be attributed to the fact that our score decomposition effectively decouples the learning of the manifold support from the density on that support, and suggests that MAD provides a partial response to the phenomenon observed by \cite{li2025scores}. Specifically, by encoding the geometry of the support through the base score, MAD bypasses the first phase of ``support recovery"
required by standard DSM, a benefit that is most evident in complex distributions (e.g., skewed discrete sets and symmetric solids).

Our results comparing MAD to ambient DSM
show that MAD leads to better proximity of the generated samples to the manifold. The base score acts as a global attractor, ensuring the generative process adheres to the manifold structure.  
This suggests that MAD may be particularly valuable in settings where the sparsity of data makes implicit manifold recovery difficult or when data is scarce.

A limitation of our approach is the reliance on analytical base scores. In future work, we plan to investigate whether approximations could serve as viable alternatives. Although we focus on low-dimensional benchmarks, our theoretical results hold for manifolds of arbitrary intrinsic dimension ($n>0$). Evaluating MAD in higher-dimensional settings and on more complex, real-world datasets -- where manifold-awareness may provide even greater benefits -- is an important direction for future research.

\newpage
\section*{Acknowledgements}
This work was supported in part by the Engineering and Physical Sciences Research Council (EPSRC) through the AI Hub in Generative Models [grant number EP/Y028805/1], and by Isambard-AI which is operated by the University of Bristol and funded by the UK Government's Department for Science, Innovation and Technology (DSIT) via UK Research and Innovation; and the Science and Technology Facilities Council (ST/AIRR/I-A-I/1023). ALJ gratefully acknowledges support from Google DeepMind and Schmidt Sciences. AP is supported by the EPSRC [2928912] and AstraZeneca via an iCASE award for a DPhil in Machine Learning. JC is supported by EPSRC through the Modern Statistics and Statistical Machine Learning (StatML) CDT programme, grant no. EP/S023151/1.
\section*{Impact Statement}
This paper presents work aimed at improving the learning of score-based generative models for manifolds. While our work may help improve generative capabilities in important domains such as drug-design, earth and climate science and potentially text generation, it may also simultaneously increase the impacts of abuse of generative capabilities. To that extent, the ethical and societal impact of our work is similar to that of other generative models and machine learning methods in general.

\bibliography{main_mad}
\bibliographystyle{icml2026}

\newpage
\appendix
\onecolumn
\section{Quaternion Representation for $\SO(3)$} \label{appx:quat_repres}
An effective way to represent rotations is to use unit quaternions (also called versors) (for a useful introduction see \cite{lynch2017modern}). A 3D rotation $R\in \SO(3)$ can
be specified by an axis of rotation given by a unit vector $u = (u_1, u_2, u_3)$ and an angle $\theta \in [0,\pi]$. It can then
be represented by exactly two versors: $x = \cos(\theta/2) + \sin(\theta/2)(u_1\textbf{i} + u_2\textbf{j} + u_3\textbf{k})$, and its antipodal
point $-x$.
The set of all versors forms the topological group $\mathcal{S}^3(1)$ and its Haar measure is the
uniform distribution over the sphere $\sigma^3$.
A 2-to-1 mapping $f: \mathcal{S}^3(1) \to \SO(3)$ transforms $x, -x$ back to their rotation, preserving their original probability distribution. 

\subsection{canonicalisation on $\SO(3)/G$ in Quaternion Coordinates}\label{appx:canon}

\begin{corollary}[Quaternion canonicalisation for $\SO(3)/G$]\label{cor:quat_canon}
Let $G\subset \SO(3)$ be a finite symmetry group acting by right multiplication. Represent
rotations by unit quaternions $q\in\mathcal{S}^3$ and represent each $g\in G$ by a unit
quaternion (with a fixed choice of sign per group element), so that the symmetry orbit of
$q$ is $\{qg:\, g\in G\}\subset\mathcal{S}^3$.
\end{corollary}

Define the canonicalisation map $\psi:\mathcal{S}^3\to\mathcal{F}$ by
\begin{equation}
    q_{\mathrm{canon}} \;=\; \psi(q)
    \;=\; \operatorname*{argmax}_{g\in G}\, \bigl|\mathrm{Re}(qg)\bigr|,
\end{equation}
where $\mathrm{Re}(q)$ denotes the scalar component of $q=(w,x,y,z)$, i.e.\ $\mathrm{Re}(q)=w$.
Finally, fix the sign ambiguity of quaternion coordinates by enforcing
\begin{equation}
    \mathrm{Re}(q_{\mathrm{canon}})\ge 0.
\end{equation}
Then $q_{\mathrm{canon}}$ is a deterministic representative of the equivalence class
$[R]\in \SO(3)/G$, and training a diffusion model on canonicalised targets removes the
symmetry-induced multimodality in $p(R\mid c)$, where $c$ is an observation such as an image, molecular group, point cloud, etc., with implicit symmetry.

At inference time, a sample may be lifted back to $\SO(3)$ by drawing $g\sim \mathrm{U}(G)$
and returning $q_{\mathrm{canon}}g$.

\paragraph{Comment.}
Maximizing $|\mathrm{Re}(qg)|$ selects the symmetry-equivalent representative closest to the
identity rotation, since for a unit quaternion $q$ the geodesic distance on $\SO(3)$ to the
identity is $2\arccos(|\mathrm{Re}(q)|)$. Note that we solely perform this operation for unit quaternions to yield $q_\text{canon}$ as our reference target for a particular condition $c$. The forward diffusion process is the standard Variance Exploding SDE.

\begin{figure}[ht]
    \centering
    \includegraphics[width=0.2\linewidth]{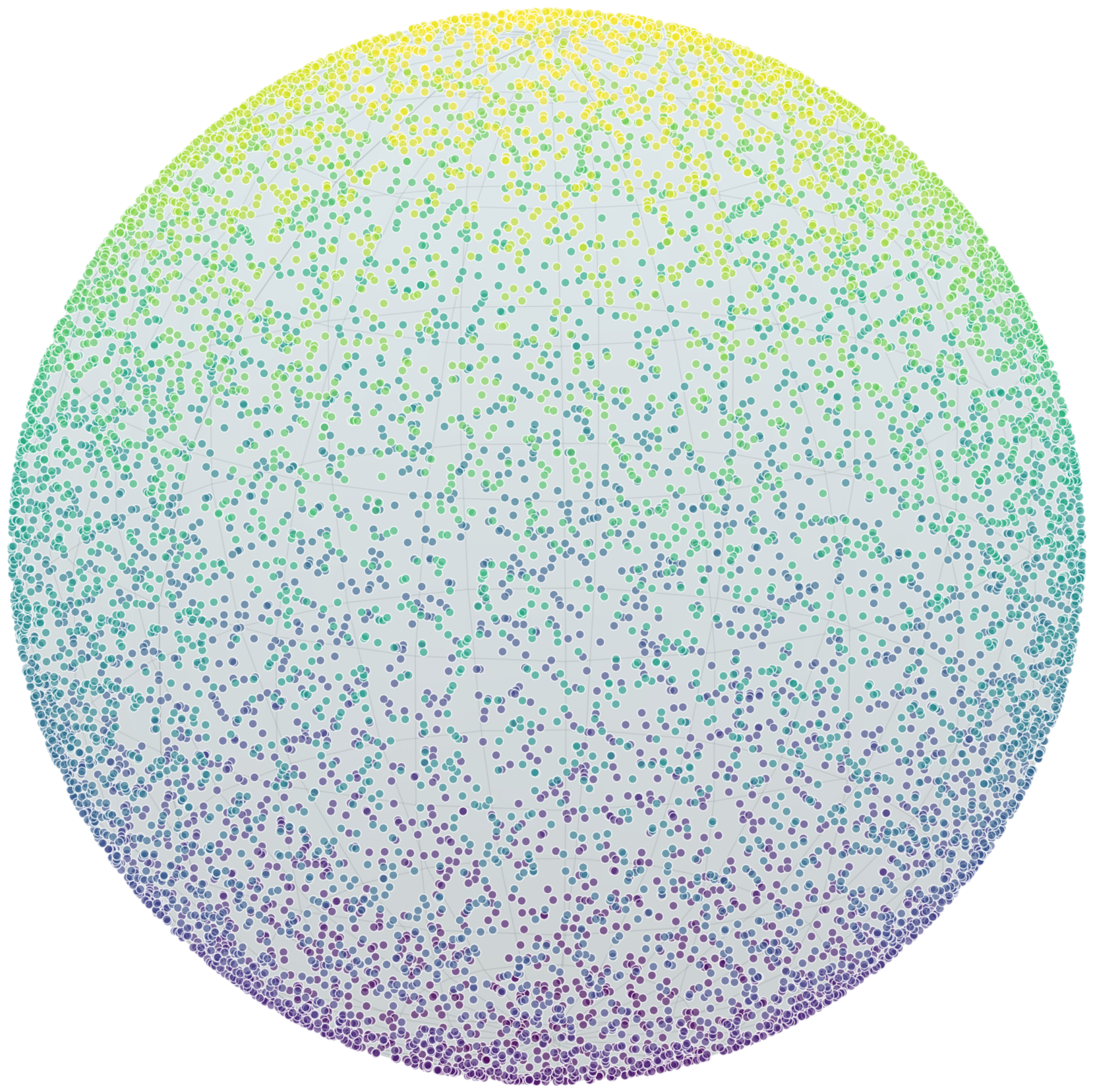}
    \caption{Demonstration of the base score for $S^2$ -- samples generated using only the base score (the score function for a uniform distribution on $\mathcal{S}^2$) with $\sigma_{\min}=1e^{-4}$. Samples are not projected onto the manifold. The colour is by $z$-axis values and is for ease of visualisation only.}
    \label{fig:s2_not_projected}
\end{figure}

\begin{figure}[ht]
  \begin{center}    \centerline{\includegraphics[width=0.4\linewidth]{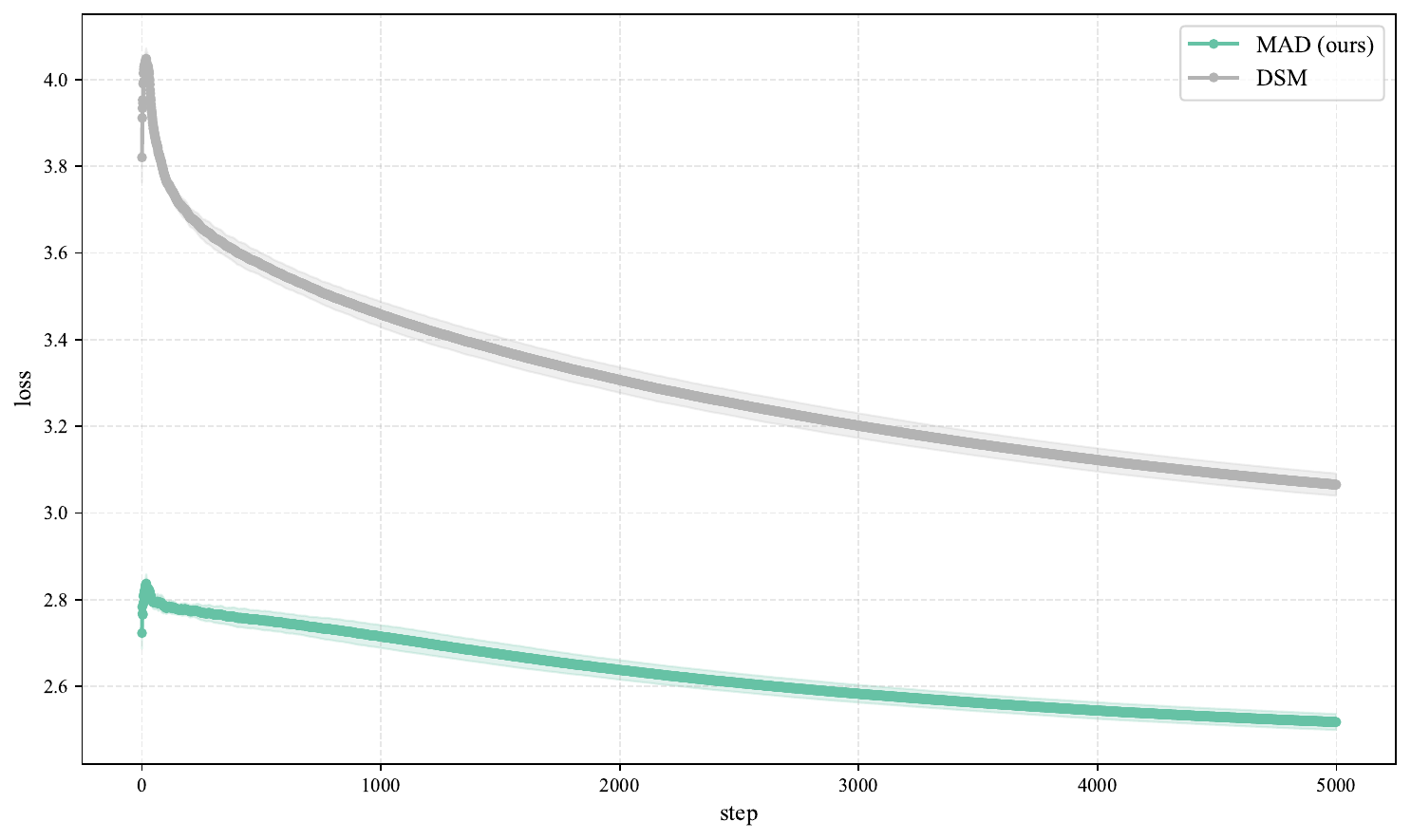}}
    \caption{
      Training loss for MAD is consistently lower throughout training compared to that of DSM. Losses are for $\SO(3)$ with $K=64$. Shaded area represents ±std across $5$ random seeds.
    }
    \label{fig:loss_dsm_vs_mad}
  \end{center}
\end{figure}

\section{Proofs} \label{appx: proofs}


\subsection{Proof of Theorem \ref{thm:o(1)_discrete}}\label{appx:proof_for_discrete}

\begin{theorem}
    Let $N \in \mathbb{N}$, $N>1$ and $\MM=\{u_i\}_{i=1}^{N}\subset\mathbb{R}^n$ for some $n \in\mathbb{N}_{>0}$. Let $\mu$ be the uniform normalised counting measure on $\mathcal{M}$ and $p$ a density function w.r.t  $\mu$ that is supported on $\MM$. Let $x \in \mathbb{R}^n$ so that  $x$ is not equidistant to all points in $\MM$. Then as $t \to 0$ ($\sigma_t \to 0$):

\begin{equation}
    \|s(x,t)-s^{base}(x,t)\| \to 0 
\end{equation}

\end{theorem}

\begin{proof}
We use the following identity which holds for all $x\in \mathbb{R}^n$:

\begin{align}\label{eq:discrete_score_to_diff_expectations}
    \|s(x,t)-s^{base}(x,t)\| &= \left\|\frac{\EE[x_0|x]-x}{\sigma_t^2} - \frac{\EE^{base}[x_0|x]-x}{\sigma_t^2}\right\| \nonumber \\
    &= \left\|\frac{\EE[x_0|x]-\EE^{base}[x_0|x]}{\sigma_t^2}\right\|
\end{align}

We therefore wish to show that:

\begin{align}\label{eq:discrete_target_abstract}
    \lim_{\sigma_t \to 0} \frac{1}{\sigma_t^2}\left\|\EE[x_0|x]-\EE^{base}[x_0|x]\right\| = 0
\end{align}

In our case:

\begin{align}
    \EE[x_0|x] &= \frac{\sum_{i=1}^N u_i N_{\sigma_t}(x-u_i) p(u_i) }{ \sum_{i=1}^N N_{\sigma_t}(x-u_i)p(u_i)}
\end{align}

\begin{align}
    \EE^{base}[x_0|x] &= \frac{\sum_{i=1}^N u_i N_{\sigma_t}(x-u_i) }{ \sum_{i=1}^N N_{\sigma_t}(x-u_i)}
\end{align}

so that our target becomes:
\begin{align}\label{eq:discrete_target}
    \lim_{\sigma_t \to 0} \frac{1}{\sigma_t^2}\left\| \frac{\sum_{i=1}^N u_i N_{\sigma_t}(x-u_i) p(u_i) }{ \sum_{i=1}^N N_{\sigma_t}(x-u_i)p(u_i)}-\frac{\sum_{i=1}^N u_i N_{\sigma_t}(x-u_i) }{ \sum_{i=1}^N N_{\sigma_t}(x-u_i)}\right\| = 0
\end{align}

The expression inside the norm can be written as:

\begin{align}
    \sum_{i=1}^N u_i \left(\frac{ N_{\sigma_t}(x-u_i) p(u_i) }{ \sum_{i=1}^N N_{\sigma_t}(x-u_i)p(u_i)}-\frac{N_{\sigma_t}(x-u_i) }{ \sum_{i=1}^N N_{\sigma_t}(x-u_i)}\right)
\end{align}

Using the triangle inequality followed by the fact that the set $\MM=\{u_i\}_{i=1}^{N}$ is bounded by $C:=\max_i \|u_i\|$, we obtain:

\begin{align}
    &\left \| \sum_{i=1}^N u_i \left(\frac{ N_{\sigma_t}(x-u_i) p(u_i) }{ \sum_{i=1}^N N_{\sigma_t}(x-u_i)p(u_i)}-\frac{N_{\sigma_t}(x-u_i) }{ \sum_{i=1}^N N_{\sigma_t}(x-u_i)}\right) \right \| \\ &\leq
   C \sum_{i=1}^N \left |  \frac{ N_{\sigma_t}(x-u_i) p(u_i) }{ \sum_{i=1}^N N_{\sigma_t}(x-u_i)p(u_i)}-\frac{N_{\sigma_t}(x-u_i) }{ \sum_{i=1}^N N_{\sigma_t}(x-u_i)} \right | 
\end{align}
so that a sufficient target would be to prove that:

\begin{align}
    \lim_{\sigma_t\to 0}\frac{1}{\sigma_t^2}
    \sum_{i=1}^N \left |  \frac{ N_{\sigma_t}(x-u_i) p(u_i) }{ \sum_{i=1}^N N_{\sigma_t}(x-u_i)p(u_i)}-\frac{N_{\sigma_t}(x-u_i) }{ \sum_{i=1}^N N_{\sigma_t}(x-u_i)} \right | = 0
\end{align}

Denoting by $\Delta_{u_i} = \frac{ N_{\sigma_t}(x-u_i) p(u_i) }{ \sum_{i=1}^N N_{\sigma_t}(x-u_i)p(u_i)}-\frac{N_{\sigma_t}(x-u_i) }{ \sum_{i=1}^N N_{\sigma_t}(x-u_i)}$, this can be further split into positive and negative terms to obtain the equivalent expression:

\begin{align}
    \lim_{\sigma_t\to 0}\frac{1}{\sigma_t^2}
    \sum_{u_i \in U^+} \Delta_{u_i}  \
    + \sum_{u_i \in U^-} -\Delta_{u_i}= 0
\end{align}

where $U^+$ and $U^-$ denote the set of points $u_i \in \MM$ for which $\Delta_{u_i}>0$ and $\Delta_{u_i}<0$, respectively. We will now bound each of these sums from above. First, observe that since the numerators are non-zero, we can write:

\begin{equation}
    \Delta_{u_i}  =  \frac{1}{1+\frac{\sum_{j\neq i}N_{\sigma_t}(u_j-x)p(u_j)}{N_{\sigma_t}(u_i-x)p(u_i)}}- \frac{1}{1+\frac{\sum_{j\neq i}N_{\sigma_t}(u_j-x)}{N_{\sigma_t}(u_i-x)}} 
\end{equation}
For $u_i \in U^+$ we can bound this from above by reducing all $p(u_j)$ to $m:=\min_{i\in[N]}p(u_i)$ to obtain: 

\begin{align}
    \Delta_{u_i} \label{eq:before_ratio_notation} \leq  \frac{1}{1+\frac{m}{p(u_i)}\frac{\sum_{j\neq i}N_{\sigma_t}(u_j-x)}{N_{\sigma_t}(u_i-x)}}- \frac{1}{1+\frac{\sum_{j\neq i}N_{\sigma_t}(u_j-x)}{N_{\sigma_t}(u_i-x)}}
\end{align}
and for $u_i \in U^-$ by increasing all $p(u_j)$ to $M:=\max_{i\in[N]}p(u_i)$ to obtain:

\begin{align}
    -\Delta_{u_i}  \leq  \frac{1}{1+\frac{\sum_{j\neq i}N_{\sigma_t}(u_j-x)}{N_{\sigma_t}(u_i-x)}}
    - \frac{1}{1+\frac{M}{p(u_i)}\frac{\sum_{j\neq i}N_{\sigma_t}(u_j-x)}{N_{\sigma_t}(u_i-x)}}
\end{align}

We now focus only on $u_i \in U^+$ since the analysis for $U^-$ is symmetric. 

For simplicity of exposition, we use the following notation for the  ratio that appears in both denominators:

$$ r_{-i, i} := \frac{\sum_{j\neq i}N_{\sigma_t}(u_j-x)}{N_{\sigma_t}(u_i-x)} $$
With this notation, and by obtaining a common denominator, the bound in Eq. \ref{eq:before_ratio_notation} becomes:

\begin{align}
     \frac{r_{-i,i}\left(1-\frac{m}{p(u_i)}\right)}{1+r_{-i,i}+\frac{m}{p(u_i)}r_{-i,i}+\frac{m}{p(u_i)}r_{-i,i}^2}= \frac{1-\frac{m}{p(u_i)}}{\frac{1}{r_{-i,i}}+1+\frac{m}{p(u_i)}+r_{-i,i}\frac{m}{p(u_i)}}
\end{align}
where the last expression follows from division by $r_{-i,i}\neq 0$.
Replacing the expression, and the one  similarly obtained for $u_i \in U^-$, into Eq. \ref{eq:before_ratio_notation} and combining with our target, we obtain that it is sufficient to prove that the resulting upper bound converges to $0$, that is:

\begin{align}
    \lim_{\sigma_t\to 0}\frac{1}{\sigma_t^2}
    \left( 
    \sum_{u_i \in U^+}  \frac{1-\frac{m}{p(u_i)}}{\frac{1}{r_{-i,i}}+1+\frac{m}{p(u_i)}+r_{-i,i}\frac{m}{p(u_i)}}
    + 
    \sum_{u_i \in U^-}  \frac{\frac{M}{p(u_i)} - 1}{\frac{1}{r_{-i,i}}+\frac{M}{p(u_i)}+1+r_{-i,i}\frac{M}{p(u_i)}}
    \right )= 0
\end{align}
Again, we focus only on the first sum since the analysis for the second is symmetric. Since the numerator is constant as a function of $\sigma_t$, it is sufficient to focus on the denominator and show that for all $i\in[N]$:

\begin{align}\label{eq:discrete_final_target}
    \lim_{\sigma_t\to 0}  \sigma_t^2\left(\frac{1}{r_{-i,i}}+r_{-i,i}\frac{m}{p(u_i)}\right) = \infty
\end{align}

Let $i\in[N]$. Observe that since exponential convergence dominates polynomial convergence, then the convergence of $r_{-i,i}$  dominates $\sigma_t^2$, and since $m\neq 0$ ($p$ is supported on $\MM$) then the above holds if $\lim_{\sigma_t\to 0} r_{-i,i} \in \{\infty, 0\}$ (the key observation is that $r_{-i,i}$ appears both as a numerator and as a denominator, so if $\lim_{\sigma_t\to 0} r_{-i,i} \in \{\infty, 0\}$ then at least one of the terms involving $r_{-i,i}$ converges to $\infty$).
We will now show that this is indeed the case.

We rewrite $r_{-i,i}$ as:

\begin{align}
    r_{-i, i} &= \sum_{j\neq i}\exp\left(\frac{-\|u_j-x\|^2+\|u_i-x\|^2}{2\sigma_t^2}\right)
\end{align}  
Consider a single term in the sum. The only case in which it does not converge to one of $\infty, 0$ as $\sigma_t \to 0$ is if the numerator is $0$ which occurs iff $\|u_j-x\|^2=\|u_i-x\|^2$, meaning $x$ is equidistant to both $u_i, u_j$, which doesn't hold by assumption.
\end{proof}
Note that the condition that $x$ is not equidistant to all points in $\MM$ is a  necessary condition. This can be seen by taking Eq. \ref{eq:discrete_target_abstract} for this case (i.e. the resulting equivalent target \ref{eq:discrete_target}) with $x$ equidistant to all points in $\MM$ (e.g. it is useful to think of $S^0\subset \mathbb{R}$ so that $x=0$). The result is the limit of a constant divided by $\sigma_t$ which tends to $\infty$ as $\sigma_t \to 0$. This shows it may be better to avoid these cases for improved numerical stability (although this should occur with probability $0$ w.r.t $p_t$ for all $t$).

\subsection{Proof of Theorem \ref{thm:base_discrete}}

\begin{theorem} (Base score for discrete distributions)
    For a discrete distribution supported on a finite set of points $\MM = \{u_i\}_{i=1}^N$, 
\begin{align}
    s^{base}(x_t,t) = \frac{\sum_{i=1}^N u_i N_{\sigma_t}(x_t-u_i) }{ \sum_{i=1}^N N_{\sigma_t}(x_t-u_i)}
\end{align}
\end{theorem}

\begin{proof}
    For a discrete distribution supported on a finite set of points $\MM = \{u_i\}_{i=1}^N$, we have,
\begin{align}
    s^{base}(x_t,t) &= \nabla_{x_t} \log \int N_{\sigma_t}(x_t-x_0) \mu(dx_0) \\
    &= \nabla_{x_t} \log \frac{1}{N}\sum_{i=1}^N N_{\sigma_t}(x_t-u_i) \\
    &= \frac{\EE^{base}[x_0|x_t] - x_t}{\sigma_t^2} \\
    \EE^{base}[x_0|x_t] &= \frac{\sum_{i=1}^N u_i N_{\sigma_t}(x_t-u_i) }{ \sum_{i=1}^N N_{\sigma_t}(x_t-u_i)}
\end{align}
\end{proof}
This is the posterior mean of the mean of the Gaussian, under a mixture of Gaussians model with means $u_i$ and variances $\sigma_t^2$.

\subsection{Proof of Theorem \ref{thm:base_sphere}}

\begin{theorem} (Base score for $n$-spheres) 
Assume $n \in \mathbb{N}_{>0}$. Let $\mathcal{M} = \mathcal{S}^n(1) := \{x \in \mathbb{R}^{n+1} : \|x\| = 1\}$ and $\mu$ be the normalized spherical measure on $\mathcal{M}$ (commonly denoted $\sigma^n$). Then for $x_t \neq 0$,
\begin{equation}
\begin{split}
    &s^{base}(x_t,t)  = -\frac{x_t}{\sigma_t^2} + \frac{1 - n}{2}\frac{x_t}{\|x_t\|^2} + \\ & \frac{1}{2I_{\frac{n-1}{2}}\left(\frac{\|x_t\|}{\sigma_t^2}\right)}\bigg( I_{\frac{n-3}{2}}\bigg(\frac{\|x_t\|}{\sigma_t^2}\bigg)+I_{\frac{n+1}{2}}\bigg(\frac{\|x_t\|}{\sigma_t^2}\bigg)\bigg)\frac{x_t}{\sigma_t^2\|x_t\|}
\end{split}
\end{equation}

where $I_k$ stands for the modified Bessel function of the first kind of order $k$.
\end{theorem}

\begin{proof}
     We derive a closed-form solution (up to the use of modified Bessel functions of the first kind) for:
    
\begin{equation}
\begin{split}
    s^{base}(x_t,t)  &=\nabla_{x_t}\log{\int_{\mathcal{S}^n(1)} N_{\sigma_t}(x_t-x_0)\mu(dx_0)}
\end{split}
\end{equation}

\begin{equation}
\begin{split}
    \int_{\mathcal{S}^n(1)} N_{\sigma_t}(x_t-x_0)\mu(dx_0)  
    &= (2\pi\sigma_t^2)^{-\frac{n+1}{2}}\int_{\mathcal{S}^n(1)} \exp{\bigg(-\frac{\|x_t-x_0\|^2}{2\sigma_t^2}}\bigg)\mu(dx_0)   \\
    &= (2\pi\sigma_t^2)^{-\frac{n+1}{2}}\int_{\mathcal{S}^n(1)} \exp{\bigg(-\frac{\|x_t\|^2-2x_t \cdot x_0+1}{2\sigma_t^2}}\bigg)\mu(dx_0) \\
    &= (2\pi\sigma_t^2)^{-\frac{n+1}{2}} \exp{\bigg(-\frac{\|x_t\|^2+1}{2\sigma_t^2}}\bigg)\int_{\mathcal{S}^n(1)} \exp{\bigg(\frac{x_t \cdot x_0}{\sigma_t^2}}\bigg)\mu(dx_0) 
\end{split}
\end{equation}

where in the second equality we used the fact that $\|x_0\|=1$.

Denote:
$$ J := \int_{\mathcal{S}^n(1)} \exp{\bigg(\frac{x_t \cdot x_0}{\sigma_t^2}}\bigg)\mu(dx_0)  = $$
We use the following proposition from \cite{mardia2009directional} (we chose to state it here for readability): 

\begin{proposition} 
Let $ k > 0 $, $\|x\| = 1 $. Then:

$$
\int_{\mathcal{S}^n(1)} e^{kx \cdot y} d\sigma(y) = \Gamma\left(\frac{n+1}{2}\right) \left(\frac{k}{2}\right)^{\frac{1-n}{2}} I_{\frac{n-1}{2}}(k)
$$

where \( I \) is the modified Bessel function of the first kind. 
\end{proposition}

Since $x_t \neq 0$, this result applies to $J$ with $k=\frac{\|x_t\|}{\sigma_t^2}$, $x=\frac{x_t}{\|x_t\|}$ and $y=x_0$ so that:

\begin{equation}
\begin{split}
    J = \Gamma\left(\frac{n+1}{2}\right) \left(\frac{\|x_t\|}{2\sigma_t^2}\right)^{\frac{1 - n}{2}} I_{\frac{n-1}{2}}\left(\frac{\|x_t\|}{\sigma_t^2}\right)
\end{split}
\end{equation}

Substituting the expression for $J$ we obtain:

\begin{equation}
\begin{split}
    \int_{\mathcal{S}^n(1)} N_{\sigma_t}(x_t-x_0)\mu(dx_0)  
    &= (2\pi\sigma_t^2)^{-\frac{n+1}{2}} \exp{\bigg(-\frac{\|x_t\|^2+1}{2\sigma_t^2}}\bigg) \Gamma\left(\frac{n+1}{2}\right) \left(\frac{\|x_t\|}{2\sigma_t^2}\right)^{\frac{1 - n}{2}} I_{\frac{n-1}{2}}\left(\frac{\|x_t\|}{\sigma_t^2}\right)
\end{split}
\end{equation}

Therefore:

\begin{equation}
\begin{split}
\nabla_{x_t}\log{\int_{\mathcal{S}^n(1)} N_{\sigma_t}(x_t-x_0)\mu(dx_0)} &= \underbrace{\nabla_{x_t}{\bigg(-\frac{\|x_t\|^2+1}{2\sigma_t^2}}\bigg)}_\text{(a)} + 
\underbrace{\nabla_{x_t}\log{\left(\frac{\|x_t\|}{2\sigma_t^2}\right)^{\frac{1 - n}{2}}}}_\text{(b)} +
\underbrace{\nabla_{x_t}\log{I_{\frac{n-1}{2}}\left(\frac{\|x_t\|}{\sigma_t^2}\right)}}_\text{(c)}
\end{split}
\end{equation}

and:
$$(a)=-\frac{x_t}{\sigma_t^2}$$

$$(b)=\frac{1 - n}{2}\bigg(\frac{1}{\|x_t\|/2\sigma_t^2}\bigg)\bigg(\frac{x_t}{2\sigma_t^2\|x_t\|}\bigg)=\frac{1 - n}{2}\frac{x_t}{\|x_t\|^2}$$

$$(c)=\frac{1}{I_{\frac{n-1}{2}}\left(\frac{\|x_t\|}{\sigma_t^2}\right)}\frac{1}{2}\bigg( I_{\frac{n-3}{2}}\bigg(\frac{\|x_t\|}{\sigma_t^2}\bigg)+I_{\frac{n+1}{2}}\bigg(\frac{\|x_t\|}{\sigma_t^2}\bigg)\bigg)\frac{x_t}{\sigma_t^2\|x_t\|}$$

\end{proof}

\subsection{Proof of Theorem \ref{thm:base_2_sphere}}

\begin{corollary} (Base score for $2$-spheres)
    For $\mathcal{M} = \mathcal{S}^2(1)=\{x \in \mathbb{R}^{3}:||x||=1\}$ and $\mu$  the usual normalised spherical measure on $\mathcal{S}^2(1)$,
\begin{equation}
\begin{split}
s^{base}(x_t,t) = x_t\left(-\frac{1}{\sigma_t^2} 
-\frac{1}{\|x_t\|^2} +
\frac{1}{\sigma_t^2\|x_t\|\tanh{\frac{\|x_t\|}{\sigma_t^2}}} \right) \\
\end{split}
\end{equation}
\end{corollary}

\begin{proof}

In this case, we can obtain a closed-form expression for $ s^{base}(x_t,t)$, without modified Bessel functions of the first kind, since $(c)$ includes only $I_{\frac{1}{2}}$, $I_{-\frac{1}{2}}$ and $I_{\frac{3}{2}}$.
The former two are known to have the closed form expressions \cite{Andrews_Askey_Roy_1999}:

\begin{equation}
    I_{\frac{1}{2}}(x)= \sqrt{\frac{2}{\pi x}}\sinh{x}\\
\end{equation}

\begin{equation}
    I_{-\frac{1}{2}}(x)= \sqrt{\frac{2}{\pi x}}\cosh{x}\\
\end{equation}

and $I_{\frac{3}{2}}$ can be obtained from them in closed form using the known recursion relation \citep{NIST:DLMF}:

\begin{equation} \label{eq:recurrence_modified_bessel}
    2nI_{n}(x)=x\left(I_{n-1}(x)-I_{n+1}(x)\right)
\end{equation} 

Specifically, we apply the recurrence to $n=\frac{1}{2}$ to obtain that for all $x\neq 0$: 

\begin{equation}
\begin{split}
I_{\frac{3}{2}}(x) &= I_{-\frac{1}{2}}(x)-\frac{1}{x}I_{\frac{1}{2}}(x) \\
&= \sqrt{\frac{2}{\pi x}}\cosh{x} - \frac{1}{x}\sqrt{\frac{2}{\pi x}}\sinh{x}
\end{split}
\end{equation}

Combining these results, $(c)$ becomes the closed form expression:

\begin{equation}
\begin{split}
(c) &= 
\frac{1}{I_{\frac{1}{2}}\left(\frac{\|x_t\|}{\sigma_t^2}\right)}\frac{1}{2}\bigg( I_{-\frac{1}{2}}\bigg(\frac{\|x_t\|}{\sigma_t^2}\bigg)+I_{\frac{3}{2}}\bigg(\frac{\|x_t\|}{\sigma_t^2}\bigg)\bigg)\frac{x_t}{\sigma_t^2\|x_t\|}\\
&= 
\frac{1}{I_{\frac{1}{2}}\left(\frac{\|x_t\|}{\sigma_t^2}\right)}\frac{1}{2}\bigg( 2I_{-\frac{1}{2}}\bigg(\frac{\|x_t\|}{\sigma_t^2}\bigg)-\frac{\sigma_t^2}{\|x_t\|}I_{\frac{1}{2}}\bigg(\frac{\|x_t\|}{\sigma_t^2}\bigg)\bigg)\frac{x_t}{\sigma_t^2\|x_t\|} \\
&= \left(\frac{\cosh{\frac{\|x_t\|}{\sigma_t^2}}}{\sinh{\frac{\|x_t\|}{\sigma_t^2}}}-\frac{\sigma_t^2}{2\|x_t\|}\right)\frac{x_t}{\sigma_t^2\|x_t\|} \\
&= \left(\frac{1}{\tanh{\frac{\|x_t\|}{\sigma_t^2}}}-\frac{\sigma_t^2}{2\|x_t\|}\right)\frac{x_t}{\sigma_t^2\|x_t\|}
\end{split}
\end{equation}

Adding this to $(a)$ and $(b)$ with $n=2$ yields the following closed form solution for all $x_t\neq 0$: 

\begin{equation}
\begin{split}
s^{base}(x_t,t) &= -\frac{x_t}{\sigma_t^2} 
-\frac{1}{2}\frac{x_t}{\|x_t\|^2} +
\left(\frac{1}{\tanh{\frac{\|x_t\|}{\sigma_t^2}}}-\frac{\sigma_t^2}{2\|x_t\|}\right)\frac{x_t}{\sigma_t^2\|x_t\|}\\
&= -\frac{x_t}{\sigma_t^2}  
-\frac{x_t}{\|x_t\|^2} +
\frac{1}{\tanh{\frac{\|x_t\|}{\sigma_t^2}}}\frac{x_t}{\sigma_t^2\|x_t\|} \\
&= x_t\left(-\frac{1}{\sigma_t^2} 
-\frac{1}{\|x_t\|^2} +
\frac{1}{\sigma_t^2\|x_t\|\tanh{\frac{\|x_t\|}{\sigma_t^2}}} \right) \\
\end{split}
\end{equation}

\end{proof}

\subsection{Proof of Theorem \ref{thm:base_3_sphere}}

\begin{corollary} (Base score for $3$-spheres)
For $\mathcal{M} = \mathcal{S}^3(1)=\{x \in \mathbb{R}^{4}:||x||=1\}$ and $\mu$  the usual normalised spherical measure on $\mathcal{S}^3(1)$,
    \begin{equation}
\begin{split}
    &s^{base}(x_t,t)  = -\frac{x_t}{\sigma_t^2} + \frac{1 - n}{2}\frac{x_t}{\|x_t\|^2} + \\ &\qquad \left( \frac{I_{0}\left(\frac{\|x_t\|}{\sigma_t^2}\right)}{I_{1}\left(\frac{\|x_t\|}{\sigma_t^2}\right)} -\frac{\sigma_t^2}{\|x_t\|}\right)\frac{x_t}{\sigma_t^2\|x_t\|}\
\end{split}
\end{equation}
\end{corollary}

\begin{proof}

Observe that the difference lies in the third term in the summation, which was denoted as expression (c) in the $n$-Recall that expression $(c)$ included the modified Bessel functions of the first kind $I_{1}, I_{0},I_{2}$. We use the recursion relation in Equation \ref{eq:recurrence_modified_bessel} to obtain the solution in terms of $I_{1}, I_{0}$. Specifically, when plugging $n=1$ into the recursion, we obtain for all $x\neq 0$:

\begin{equation}
    I_{2}(x)=I_{0}(x)-2\frac{I_{1}(x)}{x}
\end{equation}

so that $(c)$ becomes:

\begin{equation}
    \begin{split}
        (c)&=\frac{1}{I_{1}\left(\frac{\|x_t\|}{\sigma_t^2}\right)}\frac{1}{2}\bigg( I_{0}\bigg(\frac{\|x_t\|}{\sigma_t^2}\bigg)+I_{2}\bigg(\frac{\|x_t\|}{\sigma_t^2}\bigg)\bigg)\frac{x_t}{\sigma_t^2\|x_t\|}\\
        &= \frac{1}{I_{1}\left(\frac{\|x_t\|}{\sigma_t^2}\right)}\left( I_{0}\bigg(\frac{\|x_t\|}{\sigma_t^2}\bigg)-\frac{\sigma_t^2 I_1\bigg(\frac{\|x_t\|}{\sigma_t^2}\bigg)}{\|x_t\|}\right)\frac{x_t}{\sigma_t^2\|x_t\|}\\
        &=\left( \frac{I_{0}\left(\frac{\|x_t\|}{\sigma_t^2}\right)}{I_{1}\left(\frac{\|x_t\|}{\sigma_t^2}\right)} -\frac{\sigma_t^2}{\|x_t\|}\right)\frac{x_t}{\sigma_t^2\|x_t\|}\\
    \end{split}
\end{equation}
This is then added to $(a)$ and $(b)$ (with $n=3$ in $(b)$) to obtain the result.
\end{proof}

\section{Datasets}\label{appx:datasets}
For both $S^2$ earth data and $\SO(3)$ synthetic data, we first ran RSGM and saved the precise training, validation and test sets it produced during its run, for all seeds. This ensured all methods saw the same dataset. DSM, MAD and FFF were easily adapted to accept this data. 

\subsection{$\mathcal{S}^2$}\label{appx:datasets_s2}

We use earth and climate events, with the earth assumed to be a $2$-sphere. The data, as mentioned in general comments above, is used by RSGM and split by their training code into train, validation and test sets. It is then resampled during training. We save the resulting training set that arises from resampling during across all training steps. As such, we include the same events as RSGM: volcanic eruptions \cite{volcanoe_dataset}, earthquakes \cite{data_earthquake}, floods \cite{data_flood} and wild fires \cite{data_fire}. Note that the link to the flood dataset is no longer active, however this was the originally reported source for this dataset. 

\subsection{$\SO(3)$}\label{appx:datasets_SO3_synth}
We use the data generated by RSGM (it is data sampled in an online fashion where each batch is generated from a mixture of wrapped normal distributions with $K$ components). We refer the reader to the original paper for a description of how this data is generated.

\subsection{Discrete}\label{appx:data_discrete}
We generate two types -- one is uniformly sampled from 8 discrete points equally spaced on the unit circle and  the other (skewed) is using the following (we also release this in our code repository):

\[
\begin{aligned}
n_{\text{coords}} &= \lvert \text{coordinates} \rvert \\
i_{\text{peak}} &= \left\lfloor \frac{n_{\text{coords}}}{4} \right\rfloor \\
d_i &= \min\!\left( \lvert i - i_{\text{peak}} \rvert,\; 
                   n_{\text{coords}} - \lvert i - i_{\text{peak}} \rvert \right),
\quad i = 0, \dots, n_{\text{coords}} - 1 \\
p_i &= \exp(-0.8 \, d_i) \\
p_i &= \frac{p_i}{\sum_{j=0}^{n_{\text{coords}}-1} p_j} \\
\text{indices} &\sim \text{Multinomial}(p,\; n_{\text{samples}},\; \text{with replacement})
 \\
\text{data} &= \text{coordinates}[\text{indices}]
\end{aligned}
\]

\section{Training specifications}
\label{app:training}

\textbf{General notes}:

The notes hold for $\SO(3)$ synthetic (not objects), earth and discrete experiments.

In all experiments, DSM and MAD both had the same exact settings and hyperparameter tuning. The only differences were those that arise from the use of the base score, namely in the loss function and during sampling phase. 
For the best controlled settings, we adopted the same  network architecture as RSGM's for SO(3) -- a multilayer perceptron composed of five hidden layers, each with 512 neurons. 
We used the Adam optimiser as in RSGM.

All experiments on $\SO(3)$ synthetic, discrete data and earth data were run on a standard PC.

\subsection{Earth experiments}

All methods were run for 2,000 steps with a batch size of 512. 

\textbf{DSM and MAD} -- as elaborated at the start of this Section, all settings were the same save for the necessary differences. This includes hyperparameter search with $\sigma_{\min} \in \{10^{-2}, 10^{-3}, 10^{-4}, 10^{-5}, 10^{-6}\}$, the number of scales used was $100$, and learning rates $\{10^{-4}, 3 \times 10^{-4}, 5 \times 10^{-4}, 7 \times 10^{-4}, 9 \times 10^{-4}, 10^{-3}\}$.
Below are the settings that performed best for each, which were the ones used in the results section.

The best results for MAD were as follows: 




\begin{itemize}
    \item Fire: $\text{lr}=7 \times 10^{-4}$, $\sigma_{\min}=10^{-6}$
    \item Earthquake: $\text{lr}=5 \times 10^{-4}$, $\sigma_{\min}=10^{-5}$
    \item Volcano: $\text{lr}=7 \times 10^{-4}$, $\sigma_{\min}=10^{-6}$
    \item Flood: $\text{lr}=9 \times 10^{-4}$, $\sigma_{\min}=10^{-6}$
\end{itemize}

The best results for DSM were as follows:

\begin{itemize}
    \item Fire: $\text{lr}=9 \times 10^{-4}$, $\sigma_{\min}=10^{-6}$
    \item Earthquake: $\text{lr}=10^{-3}$, $\sigma_{\min}=10^{-6}$
    \item Volcano: $\text{lr}=7 \times 10^{-4}$, $\sigma_{\min}=10^{-6}$
    \item Flood: $\text{lr}=10^{-3}$, $\sigma_{\min}=10^{-6}$
\end{itemize}

\textbf{RSGM} -- we use the publicly available repository published in \cite{de2022riemannian}. 
We used setups as in the command below (also for the other values for 'experiment' (e.g. 'fire') and 'SEED' being in 0-4 incl.). For all other settings we used the defaults loaded from the associated config file. We did not perform an extensive hyperparameter sweep for this baseline, as the original work does not include guidance on which hyperparameter settings are critical, and an exhaustive search would have been computationally infeasible. We did, however, test both 'div\_free' and 'ambient' for the generators as both were specified as options, and we viewed the quality of the generator to be a potentially important factor, which was computationally feasible to test.

Command:

python main.py -m \
    experiment=earthquake \
    model=rsgm \
    generator=div\_free \
    loss=ism \
    steps=2000 \
    seed=SEED \
    val\_freq=3000

Note that since val\_freq was not used for early stopping, but only for logging information during training, we set it to a higher value than the number of steps to avoid additional computation.

\subsection{$\SO(3)$ experiments}
All methods were run for 5,000 steps with a batch size of 512, where all batches were freshly sampled using RSGM's code and saved to run the other methods on the exact same data. This amounts to $5000*512 = \sim 2.5M$ unique samples. All methods saw the same exact samples.

\textbf{DSM and MAD} -- as elaborated at the start of this Section, all settings were the same save for the necessary differences. This includes hyperparameter search with $\sigma_{\min} \in \{10^{-2}, 10^{-3}, 10^{-4}\}$, and
$\text{lr} \in \{3 \times 10^{-4}, 5 \times 10^{-4}, 7 \times 10^{-4}, 9 \times 10^{-4}, 10^{-3}\}$. The number of scales used was $100$.
Below are the settings that performed best for each, which were the ones used in the results section.

The best results for MAD were as follows:

\begin{itemize}
    \item K=16: $\text{lr}=7 \times 10^{-4}$, $\sigma_{\min}=10^{-4}$
    \item K=32: $\text{lr}=9 \times 10^{-4}$, $\sigma_{\min}=10^{-4}$
    \item K=64: $\text{lr}=7 \times 10^{-4}$, $\sigma_{\min}=10^{-4}$
\end{itemize}

The best results for DSM were as follows:

\begin{itemize}
    \item K=16: $\text{lr}=7 \times 10^{-4}$, $\sigma_{\min}=10^{-4}$
    \item K=32: $\text{lr}=9 \times 10^{-4}$, $\sigma_{\min}=10^{-4}$
    \item K=64: $\text{lr}=9 \times 10^{-4}$, $\sigma_{\min}=10^{-4}$
\end{itemize}

\textbf{RSGM} -- we used the same repository mentioned in the context of earth data. We chose the first out of two suggested close learning rates (5e-4, 2e-4) and the default setting for the other hyperparameter flow.beta. The command we used was the following (run for the three $K$ values and five seeds (SEED below in 0-4 incl.):

    python main.py -m experiment=so3 \
        model=rsgm
        dataset.K=K
        steps=5000
        loss=dsmv
        generator=lie\_algebra
        optim.learning\_rate=5e-4
        seed=SEED
        val\_freq=50

where, val\_freq was only set to the low value of $50$ for the purpose of logging train-step times every 50 steps. Otherwise it set to anything above 5000 for efficiency. See also our note in Appendix \ref{appx:how_compute_time} on how we computed train-step time to achieve the best controlled setting.

\textbf{FFF} -- after adapting the make\_so\_data function in special\_orthogonal.py to load the dataset we simply ran the following command across the different values of $K$ and seeds:

  python -m lightning\_trainable.launcher.fit \
    configs/m-fff/so3.yaml \
    data\_set.K=64 \
    data\_set.stored\_seed=SEED \
    --name 'special-orthogonal'

\subsection{SymSol: Symmetric $\SO(3)$ Experiments}
\label{app:symsol}

\subsubsection{Architecture}
We employ a conditional score-based generative model consisting of two jointly trained components: a visual encoder $E_\phi$ and a conditional denoising network $D_\theta$. We absorb these into $s_\theta$ for simplicity.

\textbf{Visual Encoder.} To condition the generation on input images, we use a ResNet-50 backbone \citep{he2015deepresiduallearningimage} initialized with weights pre-trained on the ImageNet dataset \citep{5206848}. We remove the final fully connected classification layer and extract the 2048-dimensional feature vector from the global average pooling layer. We jointly fine-tune the ResNet encoder end-to-end along with the denoiser. This allows the model to adapt the visual features specifically for the task of 3D pose estimation on gray-scale images, which also requires capturing subtle geometric cues (e.g., distinguishing the front vs. back of a cylinder) that may be invariant in standard classification tasks. A projection head maps the visual features to the hidden dimension $d_{model}$.

\textbf{Score Network parameterisation.} We parameterise the score function $s_\theta(q_t, \sigma_t, c)$ using a neural network $F_\theta$ (Residual MLP) that predicts a pre-conditioned output. Following \citet{karras2022elucidatingdesignspacediffusionbased}, we scale the network inputs to maintain unit variance via $q_{\text{in}} = q_t / \sqrt{\sigma_t^2 + 1}$. To respect the global geometry of $\SO(3)$, we must account for the double cover property where quaternions $q$ and $-q$ represent the identical rotation. This implies the score field must be parity equivariant, satisfying $s_\theta(-q_t) = -s_\theta(q_t)$. We enforce this constraint structurally by antisymmetrizing the network output:$$    F_\theta^{\text{odd}}(q_{\text{in}}, \sigma_t, c) = \frac{1}{2} \left( F_\theta(q_{\text{in}}, \sigma_t, c) - F_\theta(-q_{\text{in}}, \sigma_t, c) \right).$$The final score is obtained via $s_\theta(q_t, \sigma_t, c) = F_\theta^{\text{odd}}(q_{\text{in}}, \sigma_t, c) / \sigma_t$. We train this network using a denoising score matching loss ($\mathcal{L}_{\text{dsm}}$). With our annealing weighting schedule $\lambda(\sigma_t) = \sigma_t^2$, this parameterisation is mathematically equivalent to training $F_\theta$ to predict the negative noise $(-\epsilon)$, i.e. minimizing $\|F_\theta + \epsilon\|^2$.

\paragraph{Training}
We implement our framework using \texttt{PyTorch Lightning} to ensure reproducibility and scalable training. We train the joint architecture (ResNet encoder and Score Network) end-to-end for 128,000 steps with a batch size of 64. We employ the AdamW optimizer \citep{loshchilov2019decoupledweightdecayregularization} with a peak learning rate of $4 \times 10^{-4}$, weight decay of $1 \times 10^{-3}$, and default momentum parameters ($\beta_1, \beta_2$). To ensure stable convergence, we use a linear warmup with cosine annealing scheduler and an Exponential Moving Average (EMA) over the model weights.

A complete summary of the architecture and training hyperparameters is provided in Table \ref{tab:hyperparams}.

\begin{table}[h]
    \caption{Hyperparameter configuration for \textsc{SymSol I} experiments.}
    \label{tab:hyperparams}
    \centering
    \begin{small}
    \begin{sc}
    \begin{tabular}{ll}
        \toprule
        \textbf{Category} & \textbf{Value} \\
        \midrule
        \multicolumn{1}{l}{\textit{Optimization}} \\
        \midrule
        Optimizer & AdamW (default betas) \\
        Batch Size & 64 \\
        Training Steps & 128,000 \\
        Peak Learning Rate & $4 \times 10^{-4}$ \\
        Min Learning Rate & $1 \times 10^{-4}$ \\
        Weight Decay & $1 \times 10^{-3}$ \\
        Warmup Init Learning Rate & $1 \times 10^{-6}$ \\
        Warmup Steps & 12,800 (10\%) \\
        Gradient Clip & 10.0 \\
        Precision & 16-mixed \\
        EMA Decay & 0.999 \\
        \midrule
        \multicolumn{2}{l}{\textit{Score Network Architecture}} \\
        \midrule
        Time Embed Dim & 64 \\
        Time Encoding & Fourier \\
        Hidden Dimension ($d_{\text{model}}$) & 128 \\
        Image Embed Dim & 128 \\
        Layers & 4 \\
        \midrule
        \multicolumn{2}{l}{\textit{Diffusion Process (SDE)}} \\
        \midrule
        $\sigma_{\min}$ & 0.002 \\
        $\sigma_{\max}$ & 20.0 \\
        $\sigma_{\text{data}}$ & 0.5 \\
        Train Noise Dist. & Log-Uniform $[\sigma_\text{min}, \sigma_\text{max}]$ \\
        \bottomrule
        \bottomrule
    \end{tabular}
    \end{sc}
    \end{small}
\end{table}

\paragraph{Sampling}
At inference time, we use a polynomial noise schedule as proposed in EDM \cite{karras2022elucidatingdesignspacediffusionbased} with $\rho=4$. The sequence of noise levels is defined as:
\begin{equation}
    \sigma_i = \left( \sigma_{\max}^{\frac{1}{\rho}} + \frac{i}{N-1} (\sigma_{\min}^{\frac{1}{\rho}} - \sigma_{\max}^{\frac{1}{\rho}}) \right)^\rho.
\end{equation}
We use $N=40$ steps for all evaluations. The sampling is performed in the ambient space $\mathbb{R}^4$ using a standard Euler-Maruyama solver for the SDE adapted from \cite{karras2022elucidatingdesignspacediffusionbased}. To compute metrics employing geodesic distances, we project onto the unit sphere $\mathcal{S}^3$ to ensure the quaternion constraint is satisfied.

\textbf{Lifting.} 
The sampled outputs of the diffusion model $\hat{q}\sim p(q\mid c)$ lie in the fundamental domain $\mathcal{F}$. To recover the full distribution over $\SO(3)$, we apply a lifting procedure. We uniformly sample a symmetry element $g \sim G$ and apply the transformation $q_\text{final} = \hat{q} \cdot g$. This restores the multimodal nature of the distribution required for the correct evaluation of symmetric objects.

\section{Discrete}
We run for 2,000 steps (fresh 512-batched samples from the distribution) and perform hyperarameter tuning for both DSM and MAD with learning rates: 
3e-4, 5e-4, 7e-4, 9e-4, 1e-3 and and sigma minimum: 0.01, 0.001, 0.0001

\section{Measuring Step-Train Time} \label{appx:how_compute_time}

In each method, we wrapped the loss and step to measure the time. That is, the start time was after the batch was loaded just before the loss and train step were computed and the end time was just after. 

\end{document}